%% file: main.tex
\documentclass[journal]{IEEEtran}
\usepackage[T1]{fontenc}
 \usepackage{cite}
\usepackage{amsmath,amssymb,amsfonts}
\usepackage{amsthm}
\usepackage{graphicx}
\usepackage{algorithm}
\usepackage{algpseudocode}
\usepackage{textcomp}
\usepackage{booktabs}
\usepackage[numbers]{natbib}
\usepackage{xcolor}
\usepackage{hyperref}
\usepackage{ragged2e}
\usepackage{cuted}
\usepackage{subcaption}
\usepackage{stfloats}
\usepackage{booktabs,tabularx,array}
\newcolumntype{Y}{>{\raggedright\arraybackslash}X}
\usepackage[font=small,labelfont=normalfont]{caption}
\newcommand{\inlineeq}[2][]{%
  \refstepcounter{equation}%
  \ensuremath{#2}\ \textup{(\theequation)}%
  \if\relax\detokenize{#1}\relax\else\label{#1}\fi
}
\usepackage{soul}
\newtheorem{assumption}{Assumption}
\newtheorem{theorem}{Theorem}
\newtheorem{lemma}{Lemma}

\title{Efficient Bilevel Optimization for Meta Label Correction in Noisy Label Learning
}

\author{Ba Hoang Anh Nguyen, Viet Cuong Ta
\thanks{Ba Hoang Anh Nguyen and Viet Cuong Ta are with Human-Machine Interaction Laboratory, VNU University of Engineering and Technology, Hanoi, Vietnam. E-mail: 22021100@vnu.edu.vn, cuongtv@vnu.edu.vn.}
\thanks{Corresponding author: Viet Cuong Ta.}
}

\begin{document}

\maketitle

\begin{abstract}
Training a deep neural network with noisy labels could reduce data annotation cost but may introduce noise into the learned model.
In meta label correction approaches, an additional meta model besides the main model is trained with a small, clean dataset to correct the large, noisy dataset.
However, the update of the meta model requires the computation of hypergradients at the inner step of the main model which significantly increases the computational cost.
To improve the training efficiency, we first introduce the dynamic barrier gradient descent into standard meta label correction.
While this naive extenstion is able to speed up the training process to approximately first-order complexity, it lacks mechanisms to prevent the leakage of noisy signals to the main model and to stabilize the learning of the meta model.
Based on this observation, we propose the EBOMLC method, which is designed with three key improvements including one-step inner loop update, mixture upper loss and alignment-aware dynamic barrier.
Empirical results on CIFAR-10 and CIFAR-100 demonstrate that EBOMLC consistently outperforms other baselines, especially under high noise rate settings, while reducing training time of the meta label correction approach.
\end{abstract}

\section{Introduction}
Deep learning models have become mainstream for image classification, given sufficient training data \cite{li2021survey}.
While the architectural design and training procedures of neural networks play a significant role, the quantity and quality of the training data are central to the success of deep models.
To reduce the high data annotation cost, crowdsourcing~\cite{yan2014} or online queries~\cite{blum2003} are commonly used, which can introduce label noise into the standard data-labeling pipeline.
Without clean data, the deep neural networks easily overfit to the incorrect labels, resulting in poor performance~\cite{arpit2017}.
Therefore, an alternative is to design learning mechanisms that  mitigate noisy labels in training data~\cite{song2022}.

Methods for filtering out noisy labels can be divided into two main branches, which are sample selection and label correction.
Sample selection approaches address the noise in data by assigning weights to training samples or filtering them out during training according to label quality \cite{ren2018, shu2019}.
Label correction approaches aim to detect and correct corrupted labels to their ground truth \cite{hendrycks2018, jiang2018}.
The newly corrected samples are then combined with the clean data and used to train the main model.
In \cite{zheng2021}, the Meta Label Correction (MLC) method proposes the use of a meta model for label correction.
MLC employs a {bilevel} optimization framework, composed of an {upper loss} and a {lower loss}, to train both the main model and the meta model.
The upper loss guides the meta model to learn from a small, clean dataset and the lower loss is used to train the main model on the large, noisy training set via multiple {inner loop} updates.
This formulation enables MLC to achieve notable improvements in performance without relying on assumptions about the underlying noise distribution.

While the bilevel optimization allows MLC to effectively extract the information embedded in noisy samples, it also introduces significant computational overhead, primarily due to the need for hypergradient computation.
To address this bottleneck, recent works have explored more efficient bilevel optimization strategies \cite{liu2023, sow2022}.
Within the branch of gradient-based optimization, the Bilevel Optimization Made Easy (BOME) framework~\cite{liu2022} is able to bypass the high computation cost of hypergradient computation by employing the dynamic barrier gradient descent technique.
However, due to the presence of corrupted labels, integrating dynamic barrier into the MLC framework faces certain difficulties.
First, the large gradient step of the value function update reduces optimization efficiency and injects additional noise into the main model's update direction.
Moreover, if the signal from the noisy dataset is unstable, the main model risks to overfit to clean dataset. 
Second, the meta model receives supervision solely from the noisy set, which substantially impairs its label correction capability.
These issues result in fragile training dynamics and degrade the main model's performance in high noise rate settings.

In this work, we first extend the original MLC framework using the dynamic barrier gradient descent to achieve the first-order complexity in bi-level optimization.
By analyzing the weaknesses of this naive extension, we propose \textbf{E}fficient \textbf{B}ilevel \textbf{O}ptimization for \textbf{M}eta \textbf{L}abel \textbf{C}orrection (EBOMLC), which improves both the training efficiency and the stability of meta-label correction in {noisy label} learning.
Compared with the standard MLC with dynamic barrier, EBOMLC is introduced with three improvements.
We employ a {one-step update mechanism} in the inner loop, which maximizes the efficiency of the {first-order} framework and suppresses the dominance of noisy gradients.
The second change is a {mixture upper loss} which combines the main and meta predictions to reduce overfitting to the clean data and enhances label correction capability.
The third and most important {improvement} of EBOMLC is an alignment-aware dynamic barrier, which preserves effective label correction in the meta model and balances the influence of clean and noisy objectives.
Give these updates, EBOMLC is able to achieve stationary of the upper level objective with respect to the main parameters under the standard smoothness and boundedness assumptions.
Extensive experiments on CIFAR-10 and CIFAR-100 demonstrate that EBOMLC achieves superior accuracy compared to existing baselines across noise rate settings.
By avoiding computing the hypergradient between the upper and lower loss, it significantly reduces training time relative to the original MLC.


\input{relatedwork}

\section{Background}
\subsection{Learning with Noisy Labels Problem}
Let $\mathcal{X}$ be the feature space and $\mathcal{Y} = \{0,1\}^K$ be the ground-truth label space in a {one-hot} encoding.
Given a set of clean data $D = \{x^c,y^c\}^M$ obtained from an unknown joint distribution over $\mathcal{X} \times \mathcal{Y}$ and a set of noisy data $D' = \{x',y'\}^{N}$ obtained from a noisy joint distribution $\mathcal{X} \times \mathcal{\tilde{Y}}$  with $M \ll N$.
The main model $f$, which we aim to train and use for prediction after training, is instantiated as a function with parameters $w$ and is trained to output the predicted label $y = f_w(x)$.
The learning objective is to optimize the parameters $w$ to minimize the expected risk $E_{(x^c,y^c)\in D}[\ell(f_w(x^c),y^c]$ for some loss function $\ell$.
Since we focus on the classification task, we set $\ell$ is to the cross-entropy loss.

\subsection{Meta Label Correction (MLC)}\label{AA}
The main idea of MLC is to train the main model on the noisy dataset $D' = \{x',y'\}^N$, and with a small clean dataset $D = \{x^c,y^c\}^M$.
Furthermore, it employs an additional model, the label correction network $g$ to predict the corrected labels for the samples in $D'$.
We parameterize $g$ as $g\_alpha$, where $\alpha$ denotes the meta parameters.
The network $g_\alpha$ is typically a small neural network that takes as 
input $h(x')$, the features of $x'$ extracted by the main model, and its noisy label $y'$.
The feature extraction model $h$ is usually chosen as a standard feature-extraction backbone such as Resnet \cite{resnet32} and is frozen during training.

The MLC objective is to optimize the parameters of the main model $f_w: \mathcal{X} \to \{0,1\}^K$ to minimize the empirical loss
\begin{equation*}
w^* = \arg\min_w \frac{1}{N} \sum_{i=1}^{N} \ell(f_w(x'_i), y'_i)
\end{equation*}
with $(x'_i, y'_i) \in D'$.
To filter out noise of $y'_i$, MLC trains $g_\alpha$ to predict $\tilde{y} = g_\alpha (h(x'),y')$.
The corrected pair $(x', \tilde{y})$ is used to train the main model $f_w$.
The label correction network $g_\alpha$ is usually referred to as the meta model.

As both $f_w$ and $g_\alpha$ are trained simultaneously, the learning objective is formulated as the following optimization problem Eq. \eqref{eq:objF}:
\begin{equation}
\min_\alpha F(w^*) \hspace{0.4cm} \text{ s.t.} \hspace{0.4cm} w^*=\arg \min_w G(w,\alpha)
\label{eq:objF}
\end{equation}
where:
\begin{equation}
 F(w^*) = \frac{1}{m} \sum_{(x^c,y^c)\in D} \ell(f_w(x^c), y^c)
\label{eq:standardF}
\end{equation}
\begin{equation}
     G(w,\alpha) = \frac{1}{m'} \sum_{(x',y')\in D'} \ell(f_w(x'), g_\alpha(h(x'),y'))
     \label{eq:standardG}
\end{equation}
To distinguish between the two functions, $F$ is referred as the {upper loss} function and $G$ is the {lower loss} function.
MLC \cite{zheng2021} further introduces a {$k$-step} SGD algorithm to find the optimal parameters in Eq. \eqref{eq:objF}.
In the $k$-step inner loop, the main model parameters $w$ are updated on the noisy training dataset, using the current value of $\alpha$.
This process represents the standard parameter training phase and returns $w^*(\alpha)$.
Once the inner loop finishes, the outer loop updates the meta parameters $\alpha$ by evaluating the hypergradient $\nabla_\alpha F(w^*(\alpha))$.
This update guides subsequent inner loop updates of the main parameters $w$ in future iterations.

There are other works which employ a similar optimization scheme.
For example, L2RW~\cite{ren2018} and MW-Net~\cite{shu2019} adopt the same bilevel formulation but instantiate the lower level as a weighted training loss, letting a clean validation objective govern the weights.
In comparison to MLC, both methods compute the required hypergradients by differentiating through a single, explicitly unrolled inner update at every iteration.
Meanwhile, MLC runs the inner loop optimization for multiple steps, followed by a separate outer update.
Nevertheless, all three methods requires an estimate of the hypergradient, which is a major drawback in terms of training efficiency.

\subsection{Bilevel Optimization for MLC}
Within the optimization framework of \eqref{eq:objF}, the meta model $g_\alpha$ produces soft labels, which is fully differentiable with respect to $\alpha$.
Therefore, it is straightforward to apply the bilevel optimization framework to optimize both the main parameter $w$ and the meta parameter $\alpha$ simultaneously.
However, to compute the hypergradient $\nabla_\alpha F(w^*(\alpha))$, it is required to evaluate the Jacobian-inverse Hessian-vector product with respect to the lower objective $G(w, \alpha)$.
Computing the Hessian-vector product is computationally demanding, as it scales with the number of meta parameters $\alpha$ and main model $w$.
MLC \cite{zheng2021} proposes using a simple meta model to avoid excessive computational overhead.
Other methods rely on the value function to transform the bilevel optimization into a single-level problem \cite{liu2023, sow2022}.
This formulation does not require evaluating the hypergradient and can find the optimal parameters via penalty function methods or other techniques.
In BOME, \citet{liu2022} propose to employ dynamic barrier gradient descent to solve the reformulated problem.

In noisy label learning, the upper objective $F$ corresponds to learning on the clean dataset, while the lower objective $G$ corresponds to learning on the noisy dataset.
Bilevel optimization methods often struggle to optimize these two objectives simultaneously without proper control, which can reduce the generalization ability of deep neural networks.
Therefore, it is more straightforward to place greater emphasis on optimizing the upper objective $F$ than the lower objective $G$.
This is the key insight which motivates our design to enhance MLC based on dynamic barrier gradient descent.

\section{Motivation}

Inspired by the works of \cite{sow2022, liu2022}, we employ the value function approach to learn the optimized main parameters.
Firstly, the formulation of MLC in Eq. \eqref{eq:standardF} is rewritten into an equivalent constrained optimization
\begin{equation}
 \min_{w,\alpha} F(w) \hspace{0.6cm} s.t \hspace{0.6cm} Q(w,\alpha) := G(w,\alpha) - G^*(\alpha) \leq 0
 \label{eq:MLCandBOME}
\end{equation}
where $G^*(\alpha)=\min_w G(w,\alpha)$ is the value function.
{The key advantage of this reformulation step is to avoid the computation of the hypergradient.}

To solve this formulation, we apply the dynamic barrier gradient descent method by updating the two parameters $w$ and $\alpha$ along an appropriate update direction $\lambda$.
The direction $\lambda$ ensures a decrease in the upper loss $F$, while simultaneously controlling the constraint in Eq. \eqref{eq:MLCandBOME}.
At time step $t$ of the outer loop, $\lambda_t$ is obtained by solving a constrained quadratic optimization problem
\begin{equation}
\lambda_t=\arg\min_\lambda\|\nabla F(w^{t}) - \lambda\|^2 \text{   s.t. } \nabla Q(w^{t},\alpha^{t})^T\lambda>\phi_t
\label{eq:lambdak1}
\end{equation}
where $\phi_t = \delta \|\nabla Q(w^{t},\alpha^{t})\|^2$ is used as a dynamic barrier.
In this paper, we denote $||.||$ as the standard Euclidean norm.
The scaling value $\delta > 0$ controls the inner product between $\nabla Q$ and $\lambda$, thereby restricting the update direction of $\nabla Q$.
The solution of Eq. \eqref{eq:lambdak1} can be derived using the Lagrange
multiplier method.

It is straightforward to adapt the dynamic barrier gradient descent to MLC.
As the upper objective $F$ depends only on $w$, we extend the standard gradient $\nabla F$ in the MLC framework to include the both directions of $w$ and $\alpha$ as $\nabla F = \big(\nabla_w F,\, \nabla_\alpha F\big)$.
By the definition in Eq. \eqref{eq:standardF}, it follows that $ \nabla F = \big(\nabla_w F,\, \mathbf{0}_\alpha\big)$, where $\textbf{0}_\alpha$ is the zero vector with dimensionality equal to the number of parameters in $\alpha$.
Using this notation, the optimal direction in Eq. \eqref{eq:lambdak1} can be written as
\begin{equation}
\lambda_t = (\nabla F(w^{t}),\textbf{0}_\alpha) + \beta_t \nabla Q(w^{t}, \alpha^{t})
\label{eq:lambdak}
\end{equation}
where the scaling coefficient $\beta_t$ is computed as
\begin{equation}
\beta_t  = \max \left( \frac{\phi_t -(\nabla F(w^{t}),\textbf{0}_\alpha)^T \nabla Q(w^{t},\alpha^{t})}{\|\nabla Q(w^{t},\alpha^{t})\|^2}, 0 \right) \label{MLCDbeta}
\end{equation}

In Algorithm \ref{algorithm1}, the details of the MLC with dynamic barrier gradient descent (MLC-D) are summarized, which is divided into two main steps:
\begin{itemize}
    \item \textit{Step 1 (inner loop)}: Approximate $w^*$ by $w'$ through $k$ steps of SGD on the main model with respect to the gradient $\nabla_wG(w,\alpha)$.
    \item \textit{Step 2 (outer loop)}: Compute $Q(w,\alpha)$ = $G(w,\alpha) - G(w^*,\alpha)$ and $\beta_t$, then update both the main parameters and meta parameters by:
\begin{align*}
w^{t+1}&=w^{t} - \eta_w(\nabla_w F(w)\big|_{w=w^{t}} + \beta_t \nabla_w Q(w,\alpha) \big|_{w=w^{t}}) \\
\alpha^{t+1}&=\alpha^{t} - \eta_\alpha(\textbf{0}_\alpha + \beta_t \nabla_\alpha Q(w,\alpha) \big|_{\alpha=\alpha^{t}}) 
\end{align*}
\end{itemize}
 
\begin{algorithm}[!t]
\caption{MLC-D}
\textbf{Input:} Number of inner loops $k$ and epochs $T$, initial parameters $w^{0}, \alpha^{0}$, learning rates $\eta_w, \eta_{\alpha}$.
\begin{algorithmic}[1]
\For{$ t=0, ..., T - 1$}
    \For{$i = 0, \dots, k-1$}
        \State Sample a batch $(x',y') \sim D'$
        \State Calculate $G(w^{t}_i,\alpha^{t})$ using Eq. \ref{eq:standardG}
        \State Update $w^{t}_{i+1}\gets w^{t}_i - \eta_w \nabla_w G(w,\alpha^{t})\big|_{w=w^{t}_{i}} $
    \EndFor
    \State Compute $ Q$ following Eq. \ref{eq:MLCandBOME}
    \State $ Q(w^{t},\alpha^{t}) =  G(w^{t}_{0}, \alpha^{t})-G(w^{t}_{k-1},\alpha^{t})$
    \State Sample a batch $(x^c,y^c) \sim D$
    \State Calculate $F(w^{t})$ using Eq. \ref{eq:standardF}
    \State $\beta_t \gets \max \left( \frac{\phi_t - (\nabla F(w^{t}),\textbf{0}_\alpha)^T \nabla 
    Q(w^{t},\alpha^{t})}{\|\nabla Q(w^{t},\alpha^{t})\|^2}, 0 \right)$
    \State $\nabla_w = \nabla_w F(w)\big|_{w=w^{t}} + \beta_t \nabla_w Q(w,\alpha) \big|_{w=w^{t}}$
    \State $w^{t+1} \gets w^{t} - \eta_w \nabla_w $
    \State $\nabla_\alpha =  \beta_t \nabla_\alpha Q(w,\alpha) \big|_{\alpha=\alpha^{t}}$
    \State $\alpha^{t+1} \gets \alpha^{t} - \eta_\alpha \nabla_\alpha$
\EndFor
\end{algorithmic}
\label{algorithm1}
\end{algorithm}

The dynamics of the inner and outer loops allow the main parameters $w$ and the meta parameters $\alpha$ to achieve stationary points, as shown in \cite{liu2022}.
However, when integrating the dynamic barrier gradient descent algorithm into the MLC framework, two primary challenges arise.
First, the main parameters $w$ in MLC-D are updated by
\begin{equation*}
  w^{t+1} = w^{t} - \eta_w \big( \nabla_w F + \beta_t \nabla_w Q \big),
\end{equation*}
which can be roughly decomposed into signals from the upper loss $\nabla_w F$ and the lower loss $\nabla_w Q$.
Because the clean dataset $D$ is much smaller than the noisy dataset $D'$, 
it is difficult to balance the effects of $\nabla_w F$ and $\beta_t \nabla_w Q$.
When the magnitude $\|\nabla Q\|$ is large, this lead to the scale parameter $\beta_t$ in Eq.~\eqref{MLCDbeta} decreasing to a minor value.
The updated direction of $w^{t}$ would thus depend on the objective loss $F$ over $D$.
Thus, the main model is likely to overfit the clean training samples.
In contrast, when $\beta_t$ is large, the term $\beta_t \nabla_w Q$ introduces noisy signals from the meta model into the updates of main parameters $w$.
As a result, MLC-D can become unstable when the two terms drift out of balance, across different noise levels and training stages.
Second, the meta model of MLC-D is optimized only using the gradients $\nabla Q$ from the noisy dataset.
Therefore, when the level of noise increases, the meta model’s capacity for effective label correction is substantially reduced, which is central to the success of the original MLC.

To further analyze these drawbacks of the MLC-D algorithm, we apply Algorithm \ref{algorithm1} to learning with noisy label setting on the CIFAR-10 dataset.
The noise level is chosen at 40\%{} uniform noise.
Detailed experiment settings are provided in Section \ref{datasets}.
The loss and accuracy of both MLC and MLC-D  over 120 epochs are plotted in Fig. \ref{fig:comparison}.
In comparison to MLC, the results show a significant drop in test accuracy of MLC-D, even though it achieves good convergence of the upper loss $F$.
This suggests that the signal from the upper loss $F$ causes the main model to overfit easily, which aligns with our above analysis.
Additionally, the upper loss $F$ undergoes large fluctuations during the early stage of training.
These jumps indicate the instability of the training dynamics when MLC-D is unable to balance the two update components $\nabla F$ and $\beta \nabla Q$. 
Figure \ref{fig:mlcdheatmap} illustrates the outputs of the meta model of both MLC and MLC-D after training.
In the orginal MLC, the meta model is updated via the hypergradient {$\nabla_\alpha F(w^*( \alpha))$}.
Although $F$ does not explicitly depend on $\alpha$, hypergradient computation still propagates supervision from the clean set to the meta parameters through $w^*(\alpha)$.
Meanwhile, MLC-D updates the meta model with the gradients from $\nabla_\alpha Q$ and does not rely on {$\nabla_\alpha F(w^* (\alpha))$}.
As a result, the meta model's label correction capability is substantially reduced.

\begin{figure}[htbp]
    \centering
    \includegraphics[width=\linewidth]{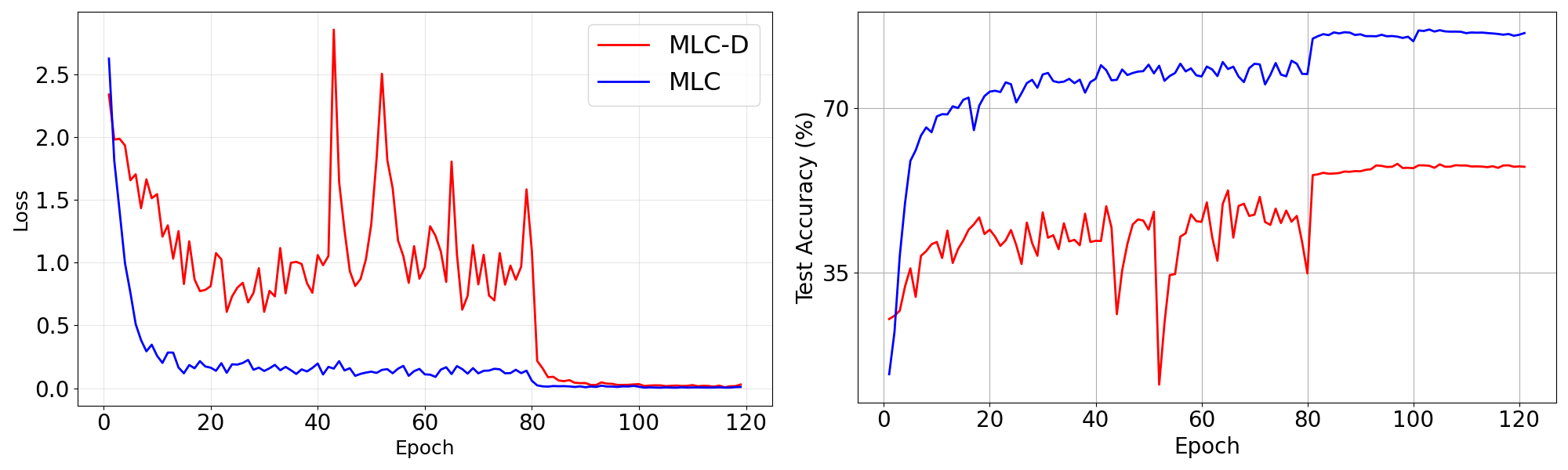}
    \caption{Training dynamics of MLC-D with 40\% uniform noise in CIFAR-10.}
    \label{fig:comparison}
\end{figure}

\begin{figure}[htbp]
    \centering
    \includegraphics[width=\linewidth]{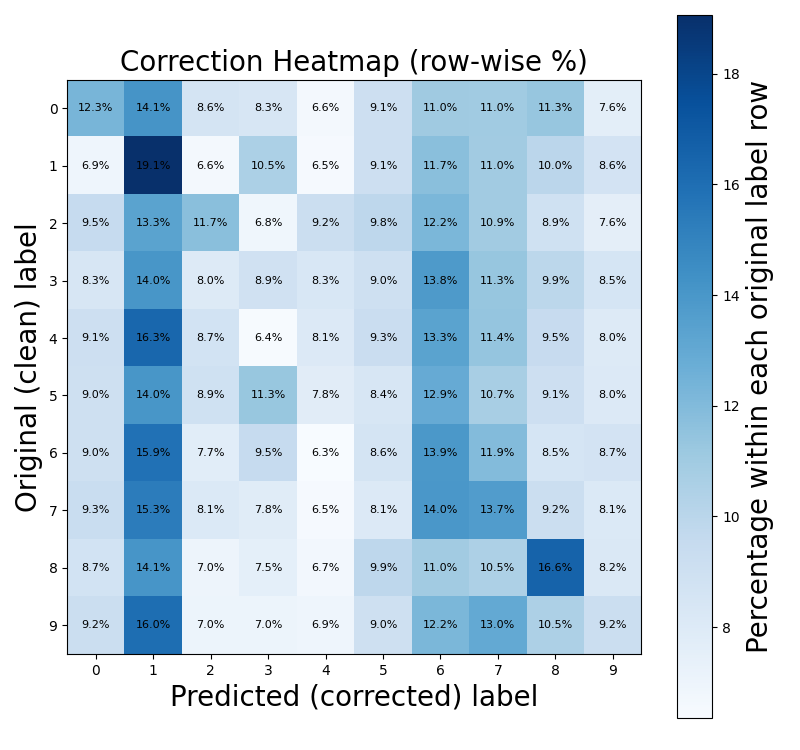}
    \caption{Heatmap of the meta model predictions of MLC-D versus clean labels with 40 \%{} uniform noise in CIFAR-10.}
    \label{fig:mlcdheatmap}
\end{figure}

\section{Method}
\subsection{Efficient Bilevel Optimization with Meta Label Correction}
Motivated by the above analysis, we propose the EBOMLC method to stabilize the training progress of  first-order {bilevel} framework under label noise.
As the step size of the inner loop can be affected by noisy samples, we employ a one-step update which both improves the efficiency and stability in training.
The upper objective is then revised and regularized using a mixture of the main model's and the meta model' output.
Finally, we update the dynamic barrier to include the misalignment direction in updating the main and the meta parameters.
These changes yield simple update rules with lower computational cost and more robustness under label noise.
We also provide theoretical analysis on the convergence of the main model under standard deep network training settings.

\subsubsection{One-step inner loop update}
In the standard BOME~\cite{liu2022} setting, the computation of $w^*$ in the inner loop ensures the convergence of the main parameters and the meta parameters.
However, this tends to make the main model focus on optimizing over the noisy dataset $D^{'}$, thereby reducing its prediction ability in general.
The over-optimization issue becomes more severe when the amount of noise increases.
To mitigate the negative effect of noisy signals from the inner loop, we reduce the computation of $w^*$ to only a single gradient step update on the value function $G$.

More specifically, given a fixed meta parameters $\alpha$, instead of the optimal point $G^*(\alpha)$ in Eq. \eqref{eq:MLCandBOME}, we alleviate the bottleneck in $k$-step look head gradient descent of $G$ by defining the new $\mathcal{Q}$:
\begin{equation}
 \mathcal{Q}(w,\alpha) = G(w,\alpha) - G(w_{(1)},\alpha)
 \label{eq:Qnew}
\end{equation}
Assuming a small learning rate update for the lower objective, it results in the following restriction on the magnitude
\begin{align*}
    \| \nabla G(w,\alpha) &- \nabla G(w_{(1)},\alpha) \| \leq\\
    \|  \nabla G(w&,\alpha) - \nabla G(w_{(T)},\alpha) \|
\end{align*}
Thus, the new $\| \nabla \mathcal{Q} \|$ is proportional to the standard $\| \nabla Q\|$.
Given arbitrary noise rates in the training data $D^{'}$, the new $\mathcal{Q}$ is expected to better prevent noisy signals from leaking into the main model.
Compared with \cite{liu2022}, our update mechanism is more general, efficient and stable than BOME in the present of noisy labels.
A main drawback of our proposed $\mathcal{Q}$ is that it cannot guarantee the optimality of the meta parameter $\alpha$. 
Nevertheless, as shown in the later Section \ref{theoryanalysis}, we still have the convergence property of the main parameters $w$, which is more essential in noisy label learning.

\subsubsection{Mixture upper loss}
Given the sparsity of data in the clean set $D$, it would causes the main model overfit easily if the model is trained directly by the gradient of upper loss $F$.
Additionally, updating the meta parameters solely from the noisy signal {$\nabla \mathcal{Q}$} degrades the capacity of the meta models for label correction.
To address these specific issues, we regularize the upper loss by introducing a new upper loss $\mathcal{F}$ as a convex mixture of the main and meta predictions:
\begin{equation}
\mathcal{F}(w, \alpha) = \frac{1}{m} \sum_{(x^c,y^c)\in D} \ell(\rho f_w(x^c) + (1-\rho)g_\alpha(h(x^c),y^c), y^c) \label{eq:newF}
\end{equation}
where $\rho$ is a constant satisfying $\rho \in (0,1)$, $f_w(x^c)$ and $g_\alpha(h(x^c),y^c)$ denote the outputs of the main model and the meta model, respectively, on the clean dataset.

In Eq. \eqref{eq:newF}, the upper loss of EBOMLC jointly supervises both the main and meta outputs toward the clean label $y_c$.
The term $g_\alpha(h(x^c),y^c)$ serves as a knowledge distillation signal, providing soft guidance that regularizes $f_w$ and mitigates overconfidence on the clean set $D$.
From an optimization perspective, our introduced $\mathcal{F}$ makes the update direction less sensitive to the main parameters $w$.
The meta model does not rely solely on the information provided by $\nabla_\alpha Q$, but is also influenced by the descent direction of $\nabla_\alpha \mathcal{F}$.
Consequently, it gains the ability to correct labels even for samples that already match the ground truth.
This helps prevent confusion in the meta model when the noisy dataset contains a mixture of clean and corrupted labels i.e. in the situation where $(x',y') \in D'$ and $y'$ matches the ground-truth label of $x'$.

The hyperparameter $\rho$ in Eq.  \eqref{eq:newF} controls how much the upper loss depends on the main model's outputs.
If $\rho$ is large, the upper loss is still dominated by the main model's prediction, and overfitting to the small clean set can persist. If $\rho$ too small, the gradient from the upper loss to the main parameters is severely attenuated ($\nabla_w \mathcal{F} \propto \rho $), effectively removing clean set supervision from the main model.


\subsubsection{Alignment-aware dynamic barrier}
From the illustrated results in Fig. \ref{fig:mlcdheatmap}, the main drawback of MLC-D comes from the noisy effects of $\| \nabla \mathcal{Q}\|$.
Moreover, the definition of the upper loss $\mathcal{F}$ implies a non-zero gradient of $\nabla_\alpha \mathcal{F}$.
Thus, this could potentially lead conflicting updates between $\nabla_\alpha \mathcal{F}$ and $\nabla_\alpha \mathcal{Q}$.
Under the setting of noisy label learning, $\nabla \mathcal{Q}$ represents the information obtained from learning with noisy labels, while $\nabla \mathcal{F}$ represents the information obtained from the clean samples.
When the conflict arises, i.e. $\nabla_\alpha \mathcal{F}^T\nabla_\alpha \mathcal{Q} <0$, the alignment between $\lambda$ and $\nabla \mathcal{Q}$ is likely to push the meta parameters $\alpha$ in a direction opposing $\nabla_\alpha \mathcal{F}$.
As a result, the conflict between $\nabla_\alpha \mathcal{F}$ and $\nabla_\alpha \mathcal{Q}$ would negatively affect the label correction capability of the meta model $g_\alpha$.

To enhance the meta network capacity, we propose a $\alpha$-gradient alignment dynamic barrier.
It incorporates both clean and noisy gradient signals from the meta network $g_\alpha$ as follows:
\begin{equation}
\bar\phi_t = \delta \| \nabla \mathcal{Q}(w^{t},\alpha^{t})\|^2 + \nabla_\alpha \mathcal{F}(w^{t},\alpha^{t})^T\nabla_\alpha \mathcal{Q}(w^{t},\alpha^{t})
\label{newPhi}
\end{equation}
where $\nabla_\alpha \mathcal{F}(w^{t},\alpha^{t})^T\nabla_\alpha \mathcal{Q}(w^{t},\alpha^{t})$ is referred to as the alignment term.
When the alignment term is positive, it allows $\lambda$ to be aligned with $\nabla \mathcal{Q}(w^{t},\alpha^{t})$ while still preserving the label correction capability of the meta model.
Otherwise, the constraint in Eq. \eqref{eq:lambdak1} is relaxed, allowing $\lambda$ to prioritize approaching $\nabla \mathcal{F}$ rather than forcing alignment with $\nabla \mathcal{Q}$ as in the previous barrier function.
Note that, the alignment term is specific to our proposed EBOMLC. 
In case of MLC-D, the upper loss depends only on the main model.
Therefore, the gradient $\nabla_\alpha \mathcal{F}$ is always zero and the update direction $\lambda$ in Eq. \eqref{eq:lambdak} is always aligned with $\nabla \mathcal{Q}$. 

With this modified barrier function, the optimization problem in Eq. \ref{eq:lambdak} yields a new solution as given by {the update direction}: 
\begin{equation*}
\bar\lambda_t = \nabla \mathcal{F}(w^{t}, \alpha^{t}) + \bar\beta_t \nabla \mathcal{Q}(w^{t}, \alpha^{t})
\end{equation*}
where 
\begin{equation}
\bar\beta_t  = \max \left( \delta -\frac{ \nabla_w \mathcal{F}(w^{t},\alpha^{t})^T \nabla_w \mathcal{Q}(w^{t},\alpha^{t})}{\|\nabla \mathcal{Q}(w^{t},\alpha^{t})\|^2}, 0 \right) \label{newbeta}
\end{equation}
When the upper loss $\mathcal{F}$ is used with $\rho=1$, the component $\nabla_\alpha \mathcal{F} = 0$.
Thus, the proposed barrier function reduces to the standard barrier in Eq. \eqref{eq:lambdak}.

The decrease in $\| \nabla \mathcal{Q} \|$ from a one-step update causes $\bar\beta$ to become large.
Furthermore, because the regularized upper loss reduces the influence of clean signal $\nabla \mathcal{F}$, it is necessary to control the impact of noisy signals $\nabla \mathcal{Q}$ in the update direction $\bar\lambda$.
Specifically, we scale $\bar\beta$ by a factor $\xi$, which plays a role similar to a learning rate of the noisy signal, directly reducing the influence of $\nabla \mathcal{Q}$ in the update:
\begin{equation}
\bar\lambda_t = \nabla \mathcal{F} + \underbrace{\xi \cdot \bar \beta_t}_{\text{reduced weight}}\nabla\mathcal{Q}
\label{newLambda}
\end{equation}
In the above equation, the constant $\xi$ is constrained to a value $\in (0,1)$.
to ensure that the vector $\xi\bar\beta_t \nabla \mathcal{Q}$ is a scaled-down version of $\bar\beta_t \nabla \mathcal{Q}$ while preserving its direction.

\subsection{Implementation Details}
Given the one-step update in Eq. \eqref{eq:Qnew}, the mixture upper loss $\mathcal{F}$ in Eq. \eqref{eq:newF} and the derivation of $\bar\lambda_t$ in Eq. \eqref{newLambda}, it is straightforward to train the main model and the meta model by following the MLC-D setup.
The update direction $\bar \lambda$ is decomposed as $\nabla_w$ and $\nabla_\alpha)$, aligned with $w$ and $\alpha$, respectively:
\begin{equation}
\begin{aligned}
\nabla_w = &\nabla_w \mathcal{F}(w,\alpha)\big|_{w=w^{t}} + \xi\bar\beta_t \nabla_w \mathcal{Q}(w,\alpha) \big|_{w=w^{t}}\\
w^{t+1} &\gets w^t - \eta_w \nabla_w 
    \label{eq:ourupdatew}
\end{aligned}
\end{equation}
\begin{equation}
\begin{aligned}
\nabla_\alpha =& \nabla_\alpha \mathcal{F}(w,\alpha)\big|_{\alpha=\alpha^{t}} + \xi\bar\beta_t \nabla_\alpha \mathcal{Q}(w,\alpha) \big|_{\alpha=\alpha^{t}}\\
\alpha^{t+1} &\gets \alpha^{t} - \eta_\alpha \nabla_\alpha
    \label{eq:ourupdatea}
\end{aligned}
\end{equation}
where $\eta_w$ is the the learning rate of the main model $f_w$ and $\eta_\alpha$ is the learning rate of the meta model $g_\alpha$.
Technically, we assume that the meta model is updated more slowly than the main model.
Therefore, the main model can adapt to the new labels from the meta model.
Otherwise, each update of the meta model would make the learning targets of the main model unstable, hurting the training dynamics of both models.
Under such conditions, we are able to provide a convergence guarantee for the main model in the next section.

Figure \ref{Figure 1} provides an illustration of our EBOMLC update mechanism in Eq. \ref{eq:ourupdatew} and Eq. \ref{eq:ourupdatea}.
The detailed implementation is given in Algorithm \ref{alg2}.
In practice, we choose $\xi=0.5, \rho=0.2$, and $\delta = 0.25$ to achieve the best empirical performance based on experimental results.

\begin{figure}[!t]
    \centering
    \includegraphics[width=0.5\textwidth]{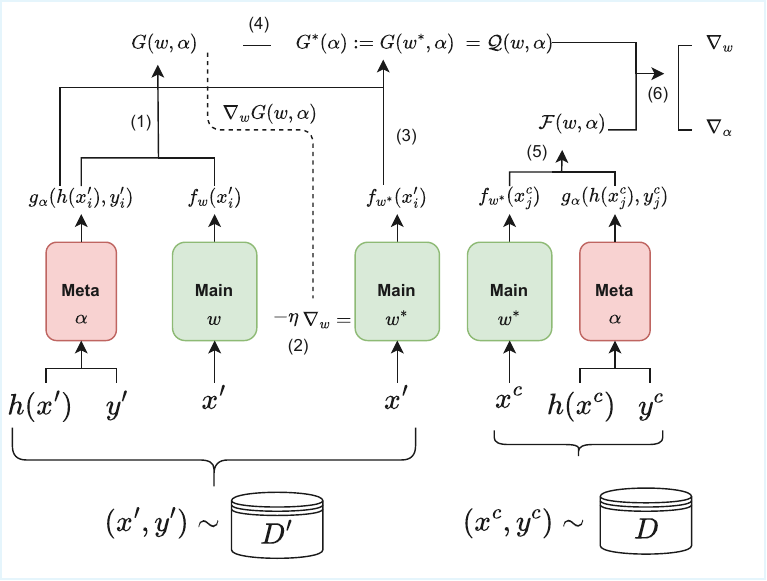}
    \caption{This figure illustrates the computational graph of our method. The procedure is as follows: (1) evaluating the loss between the main model’s prediction and the soft label corrected by the meta model on the noisy data \((x', y') \sim D'\); (2) computing \( w_{(1)} \) via a one-step update from the current \( w \) according to the gradient $\nabla_w G(w,\alpha)$; (3) feeding the feature \( x' \) to main model to compute \( G(w_{(1)},\alpha) \); (4) deriving \( \mathcal{Q}(w,\alpha) \) based on Eq. \eqref{eq:Qnew}; (5) calculating the loss between the main model's output and the label corrected by the meta model on the clean minibatch \((x^c, y^c) \sim D\); and finally, (6) computing $\bar\beta$ according to the Eq. \eqref{newbeta} to obtain $\nabla_w$, $\nabla_\alpha$, then updating both the main and meta model parameters.}
    \label{Figure 1}
\end{figure}

\algrenewcommand\algorithmiccomment[1]{\hfill$\triangleright$ #1}
\begin{algorithm}
\caption{EBOMLC}
\textbf{Input}: Clean and noisy datasets $D,D'$, number of training steps $T$, initial parameters $w^{0}, \alpha^{0}$, learning rates $\eta_w, \eta_{\alpha}$.\\
\textbf{Output}: Optimized parameters $w^{T}, \alpha^{T}$.
\begin{algorithmic}[1]
\For{$ t=0, ..., T - 1$}
    \State Sample a batch $(x',y') \sim D'$
    \State Calculate $G(w^{t},\alpha^{t})$ using Eq. \ref{eq:standardG} 
    \State Get $w^t_{(1)} \gets w^{t} - \eta_w \nabla_w G(w^t,\alpha)$ 
    \State $\mathcal{Q}(w^{t},\alpha^{t}) = G(w^{t}, \alpha^{t})-G(w^t_{(1)},\alpha^{t})$
    \State Sample a batch $(x^c,y^c) \sim D$
    \State Calculate $\mathcal{F}(w^{t},\alpha^{t})$ using Eq. \ref{eq:newF} 
    \State $\bar\beta_t \gets \max \left( \delta - \frac{  \nabla_w \mathcal{F}(w^{t},\alpha^{t})^T \nabla_w 
    \mathcal{Q}(w^{t},\alpha^{t})}{\|\nabla \mathcal{Q}(w^{t},\alpha^{t})\|^2}, 0 \right)$
    \State $\nabla_w = \nabla_w \mathcal{F}(w,\alpha)\big|_{w=w^{t}} + \xi\bar\beta_t \nabla_w \mathcal{Q}(w,\alpha) \big|_{w=w^{t}}$
    \State $w^{t+1}=w^{t} - \eta_w \nabla_w $
    \State  $\nabla_\alpha = \nabla_\alpha \mathcal{F}(w,\alpha)\big|_{\alpha=\alpha^{t}} + \xi\bar\beta_t \nabla_\alpha \mathcal{Q}(w,\alpha) \big|_{\alpha=\alpha^{t}}$
    \State $\alpha^{t+1}=\alpha^{t} - \eta_\alpha \nabla_\alpha$
\EndFor

\end{algorithmic}
\label{alg2}
\end{algorithm}




\subsection{Theoretical Analysis}
\label{theoryanalysis}

We further provide the convergence analysis for the upper loss $\mathcal{F}(w, \alpha)$ using the update rule by $\nabla_w$ in Eq. \eqref{eq:ourupdatew} and $\nabla_\alpha$ in Eq. \eqref{eq:ourupdatea}.
The settings of our analysis are built upon the smoothness and bounded assumptions as follows:

\begin{assumption}[Lipschitz Conditions on Gradients]\label{assumption1}
The upper loss functions $\mathcal{F}(w,\alpha)$ and lower loss $G(w,\alpha)$ are differentiable functions and 
their gradients $\nabla \mathcal{F}$ and $\nabla G$ are L-Lipschitz with respect to the combined input variables $(w,\alpha)$, for some constant $L>0$. 
\end{assumption}

\begin{assumption}[Bounding Gradients]\label{assumption2}
There exists a finite constant \( M \) such that for any \( (w, \alpha) \), the gradients of \( \mathcal{F} \) and \( G \), as well as the function values \( \mathcal{F}(w, \alpha) \) and \( G(w, \alpha) \), are all bounded by \( M \) in magnitude.
\end{assumption}

Under the above assumptions, we establish the sublinear convergence of the upper loss $\mathcal{F}$ with respect to the main parameters $w$, as stated in Theorem \ref{theorem1}.
\begin{theorem}\label{theorem1}
 Fix the total number of iterations \( T \). Let the learning rate for the meta parameters $\alpha$ be defined as 
$\eta_\alpha = \frac{c}{T}$ where \( c > 0 \) is a constant satisfying \( \frac{c}{T} < 1 \). Let the learning rate for the main model parameters $w$ be defined as $\eta_w = \frac{C}{\sqrt{T}}$,
where the constant \( C \) satisfies
\begin{equation*}
\dfrac{C}{\sqrt{T}} <\min( \dfrac{1-\xi}{L},1)
\end{equation*} 
Then, the following inequality holds:
\begin{equation*}
\min_{1 \leq t \leq T} \left\| \nabla_w \mathcal{F}(w^{(t)}, \alpha^{(t)}) \right\|^2 \leq \mathcal{O} \left( \frac{H}{\sqrt{T}} \right),
\end{equation*}
where \( H \) is a constant independent of the convergence process.
\end{theorem}
The sublinear convergence rate is well expected under the bilevel optimization framework.
For example, in the case of the naive extension MLC-D, one can adapt the proofs in BOME~\cite{liu2022}
to obtain a similar convergence rate of the original upper objective $F$.
However, as the BOME framework seeks to have the optimality of both the upper loss and lower losses, it further requires an additional Polyak–Łojasiewicz condition on the lower loss function.
As discussed above, the convergence of meta parameters is likely harmful to the label correction of the meta model, especially when the noise rate in $D'$ is high.
Therefore, our analysis in EBOMLC takes a different route and focuses only on the convergence property of the upper loss function, which connects to the main parameters.


In noisy label learning, 
it is expected that the learning rate of the meta parameters $\alpha$ be slower  than for the main parameters. Building on this intuition, we further introduce Assumption \ref{assumption3} on the learning rate $\eta_\alpha$, given the learning rate $\eta_w$ of the main model $w$.
\begin{assumption} \label{assumption3}
Let \( \eta_w^{(t)} \) and \( \eta_\alpha^{(t)} \) denote the learning rates of the main model parameters $w$ and the meta model $\alpha$ parameters at iteration \( t \), respectively. Assume the meta learning rates $\eta_\alpha^{(t)}$ form a convergent series and the following conditions hold:
\begin{equation*}
\sum_{t=1}^{\infty} \eta_w^{(t)} = \infty, \quad
\sum_{t=1}^{\infty} \left( \eta_w^{(t)} \right)^2 < \infty.
\end{equation*}
\begin{equation*}
\eta_\alpha^{(t)} <\eta_w^{(t)}<\min( \dfrac{1-\xi}{L},1), \quad t=1,2,...
\end{equation*}
\end{assumption}

Given the above assumptions, the next theorem offers a stronger convergence guarantee for the main model.
\begin{theorem}{\label{theorem2}}
Then, the gradient of the upper objective with respect to the main model parameters $w$ converges to zero:
\[
\lim_{t \to \infty} \left\| \nabla_w \mathcal{F}(w^{(t)}, \alpha^{(t)}) \right\| = 0
\]
\end{theorem}
Specifically, Theorem \ref{theorem2} ensures that the main model parameters $w$ converge to a local minimizer associated with a terminal value of the meta parameters $\alpha$.
It can be achieved with a meta learning rate schedule whose series is summable, resulting in soft label correction effect of the meta model $g_\alpha$.
As pointed out by \citet{zheng2021}, soft labels produced by the meta model can provide richer supervision for the main model.
Therefore, EBOMLC deliberately sacrifices convergence of the meta model to a better-performing main model.
In practical settings, the convergence of our modified $\mathcal{F}$ to the optimized $w^*$ is expected to be associated with the convergence of the main model $f$ on the clean set.
We provide a detailed analysis to this behavior in Section \ref{analysis} by varying the mixture coefficient $\rho$.



\begin{table*}[htbp]
    \centering
    \begin{tabular}{l|ccccc|cc|c|c}
        \toprule
        & \multicolumn{7}{c|}{CIFAR10} & \textbf{Avg.} & \textbf{Training} \\
        \textbf{Method}
        & \textbf{uniform-0\%} & \textbf{uniform-20\%} & \textbf{uniform-40\%} & \textbf{uniform-60\%} & \textbf{uniform-80\%}
        & \textbf{flip-20\%} & \textbf{flip-40\%}
        & \textbf{Rank} & \textbf{Time} \\
        \midrule
        Co-teaching+ & 89.14 & 85.79 & 82.98 & 56.37 & 44.85 & 77.56 & 72.93 & 7.43 & $0.51 \times$ \\
        DivideMix    & 91.32 & 87.15 & 85.16 & \underline{82.45} & \textbf{77.69} & 82.90 & 78.38 & 4.43 & $1.12 \times$ \\
        GLC          & 92.25 & \underline{89.31} & 86.20 & 74.27 & 64.38 & \textbf{90.97} & \underline{89.75} & \underline{3.00} & $0.33 \times$ \\
        L2RW         & 85.16 & 82.51 & 78.01 & 72.19 & 55.81 & 84.49 & 80.86 & 6.57 & $1.69 \times$ \\
        MW-Net       & 88.91 & 80.61 & 71.55 & 64.02 & 48.36 & 87.11 & 84.75 & 6.86 & $1.89 \times$ \\
        PMW-Net      & 90.74 & 88.94 & \underline{86.80} & 81.56 & 45.30 & 88.79 & 87.36 & 4.29 & $1.95 \times$ \\
        MLC          & \underline{92.76} & 89.01 & 86.08 & 82.08 & 74.10 & 89.39 & 87.12 & 3.14 & $2.28 \times$ \\
        MLC-D        & 67.75 & 66.84 & 67.56 & 67.14 & 67.58 & 66.18 & 62.57 & 8.00 & $2.32 \times$ \\
        EBOMLC       & \textbf{92.92} & \textbf{92.18} & \textbf{89.58} & \textbf{83.51} & \underline{76.02} & \underline{90.74} & \textbf{90.54} & \textbf{1.29} & $1.00 \times$ \\
        \bottomrule
    \end{tabular}
    \vspace{0.5em}
    \caption{Comparison of test accuracy on \textsc{CIFAR10} under uniform and flip noise. The best and second-best results are marked in \textbf{bold} and \underline{underline}. The \textbf{Avg. Rank} column denotes the average ranking across seven noise settings (lower is better).}
    \label{cifar10}
\end{table*}

\begin{table*}[htbp]
    \centering
    \begin{tabular}{l|ccccc|cc|c|c}
        \toprule
        & \multicolumn{7}{c|}{CIFAR100} & \textbf{Avg.} & \textbf{Training} \\
        \textbf{Method}
        & \textbf{uniform-0\%} & \textbf{uniform-20\%} & \textbf{uniform-40\%} & \textbf{uniform-60\%} & \textbf{uniform-80\%}
        & \textbf{flip-20\%} & \textbf{flip-40\%}
        & \textbf{Rank} & \textbf{Time} \\
        \midrule
        Co-teaching+ & 56.28 & 50.49 & 46.80 & 34.96 &  9.14 & 49.15 & 36.69 & 7.57 & $0.51 \times$ \\
        DivideMix    & 66.34 & 61.92 & 56.14 & \underline{48.59} & 11.18 & 54.14 & 48.12 & 4.43 & $1.12 \times$ \\
        GLC          & \textbf{68.28} & \underline{63.48} & \underline{56.82} & 46.80 & \underline{23.27} & \textbf{66.64} & 47.59 & \underline{2.57} & $0.33 \times$ \\
        L2RW         & 58.84 & 52.90 & 45.19 & 37.71 & 18.96 & 53.57 & 48.59 & 5.71 & $1.69 \times$ \\
        MW-Net       & 63.91 & 55.07 & 52.55 & 45.17 & 18.87 & 58.99 & 45.41 & 5.14 & $1.89 \times$ \\
        PMW-Net      & 66.98 & \textbf{63.52} & 56.12 & 47.12 & 20.53 & 56.12 & 47.73 & 3.43 & $1.95 \times$ \\
        MLC          & 65.03 & 48.73 & 39.06 & 25.26 & 15.58 & 58.98 & \underline{57.46} & 5.86 & $2.28 \times$ \\
        MLC-D        & 17.02 & 13.26 & 15.98 & 14.91 & 12.06 & 15.47 & 15.58 & 8.71 & $2.32 \times$ \\
        EBOMLC       & \underline{67.13} & 62.92 & \textbf{58.15} & \textbf{50.04} & \textbf{35.41} & \underline{64.86} & \textbf{58.11} & \textbf{1.57} & $1.00 \times$ \\
        \bottomrule
    \end{tabular}
    \vspace{0.5em}
    \caption{Comparison of test accuracy on \textsc{CIFAR100} under uniform and flip noise. The best and second-best results are marked in \textbf{bold} and \underline{underline}. The \textbf{Avg. Rank} column denotes the average ranking across seven noise settings (lower is better).}
    \label{cifar100}
\end{table*}

\section{Experiments}
\subsection{Settings}
Under the bilevel optimization framework for noisy label learning, we design our experiments to follow the standard settings of MLC~\cite{zheng2021}.
Moreover, to highlight the advantages of our proposed EBOMLC, we use the default settings for the noise distribution, network architecture and training loss.
The detailed of these settings are given as follows.
\subsubsection{Datasets}
\label{datasets}
We evaluate our proposed EBOMLC and other baselines on the two image datasets: CIFAR-10 and CIFAR-100.
Each dataset consists of 50,000 training images and 10,000 test images. 
There are 10 classes in CIFAR-10 and 100 classes in CIFAR-100.
Following the prior work ~\cite{shu2019,zheng2021}, we use $2 \%{}$ of the training images as the clean data set $D$, and the remaining $98 \%{}$ as the noisy dataset $D'$.
The number of training samples per class are distributed equally in both the clean and noisy training dataset.

The noisy training dataset $D'{}$ is created by generating incorrect labels following either a symmetric or an asymmetric distribution.
In the symmetric (uniform) noise setting, the true label is randomly flipped uniformly to another class, regardless of the true label. 
We select the percentage of the noisy samples in uniform noise in the list of of $[0 \%{}, 20\%{}, 40\%{}, 60\%{}, 80\%{}]$ of the total samples in $D'{}$.
In the asymmetric (flip) noise setting, the true label are flipped to a specific incorrect class with higher probability, often mimicking real-world confusion or systematic labeling errors.
We only test with the percentage values of $[20\%{}, 40\%{}]$ in the flip noise setting. We report the mean performance over 3 random seeds.

\subsubsection{Models and Baselines}
Given the training setup with a small clean dataset $D$  and a large noisy dataset $D'{}$, we compare our proposed EBOMLC with standard MLC \cite{zheng2021} and its variant MLC-D from Algorithm \ref{algorithm1}. 
We also include the six baselines, which relate to the bilevel optimization, including:
\begin{itemize}
    \item \textbf{Co-teaching+}~\cite{yu2019}: Co-teaching+ trains two networks simultaneously and lets them exchange small-loss examples, further reducing the risk of memorizing noisy labels.
    \item \textbf{DivideMix}~\cite{dividemix2020}: DivideMix treats learning with noisy labels as a semi-supervised problem by splitting data into clean/noisy subsets via a Gaussian mixture model on the losses. It remains one of the strongest and most widely used baselines for noisy-label learning.
    \item \textbf{GLC}~\cite{hendrycks2018}: GLC employs a noise transition matrix to correct the noisy data based on the clean training dataset.
    \item \textbf{L2RW}~\cite{ren2018}: L2RW introduces online training weights to train samples based on feedback from a clean set.
    \item \textbf{MW-Net}~\cite{shu2019}: MW-Net further expands the reweighting approach of L2RW by learning an explicit nonlinear mapping to automatically assign weights to training samples.
    \item \textbf{PMW-Net}~\cite{zhao2021}: PMW-Net further generalizes MW-Net by adopting a probabilistic formulation of the weighting function, which improves the robustness and stability of meta reweighting under label noise.
\end{itemize}
{Note that, the baselines are selected to work with both low and high noise rate settings, and without the requirement of a strong underlying representation learning model. Moreover, some methods employ different backbone architectures~\cite{wang2021, xu2025dulc} or self-pretrained strategies~\cite{emlc2023, dmlp2023}, making it challenging to disentangle whether the performance improvements originate from the proposed methods or from the backbone networks.}

For a fair comparison, we used the same small main model architecture, Resnet-32 \cite{resnet32}, for both the CIFAR-10 and CIFAR-100 datasets, similar to studies using the bilevel framework~\cite{shu2019,zhao2021, zheng2021}.
The details implementation of the three methods MLC, MLC-D and EBOMLC involves a further meta model.
In our implementation, the backbone $h$ is a frozen ResNet-32 feature extractor, and we use an embedding layer of size $(C, 128)$ to embed the input noisy labels where $C$ denotes the number of classes of the corresponding dataset.
The meta model $g_\alpha$ takes as input the example's feature $h$ together with the noisy label embedding and outputs a categorical distribution over classes as the corrected labels.
We employ a three-layer feedforward network for $g_\alpha$.
In MLC-D, we apply $k=5$ steps for the inner loop to extract the gradient $\nabla Q$ from the lower objective $G$.

All models are trained with the standard cross entropy loss function.
The total number of training epochs is 120
We use the SGD method with the momentum set to 0.9, weight decay set to $5e-4$ for the main model. 
The main model is trained with a learning rate of 0.1, which is decayed by a factor of 0.1 at epochs 80 and 100.
We use a smaller learning rate $3e-4$ for the meta model with Adam optimizer.
There is no additional regularization techniques or data augmentation methods in the training setup.
All of the experiments are run on T4 GPU for a fair comparision of computing speed.
Our EBOMLC codebase will be made available online upon the acceptance of the manuscript.

\begin{figure*}[!t]
    \centering
    \begin{subfigure}[b]{0.48\textwidth}
        \includegraphics[width=\linewidth]{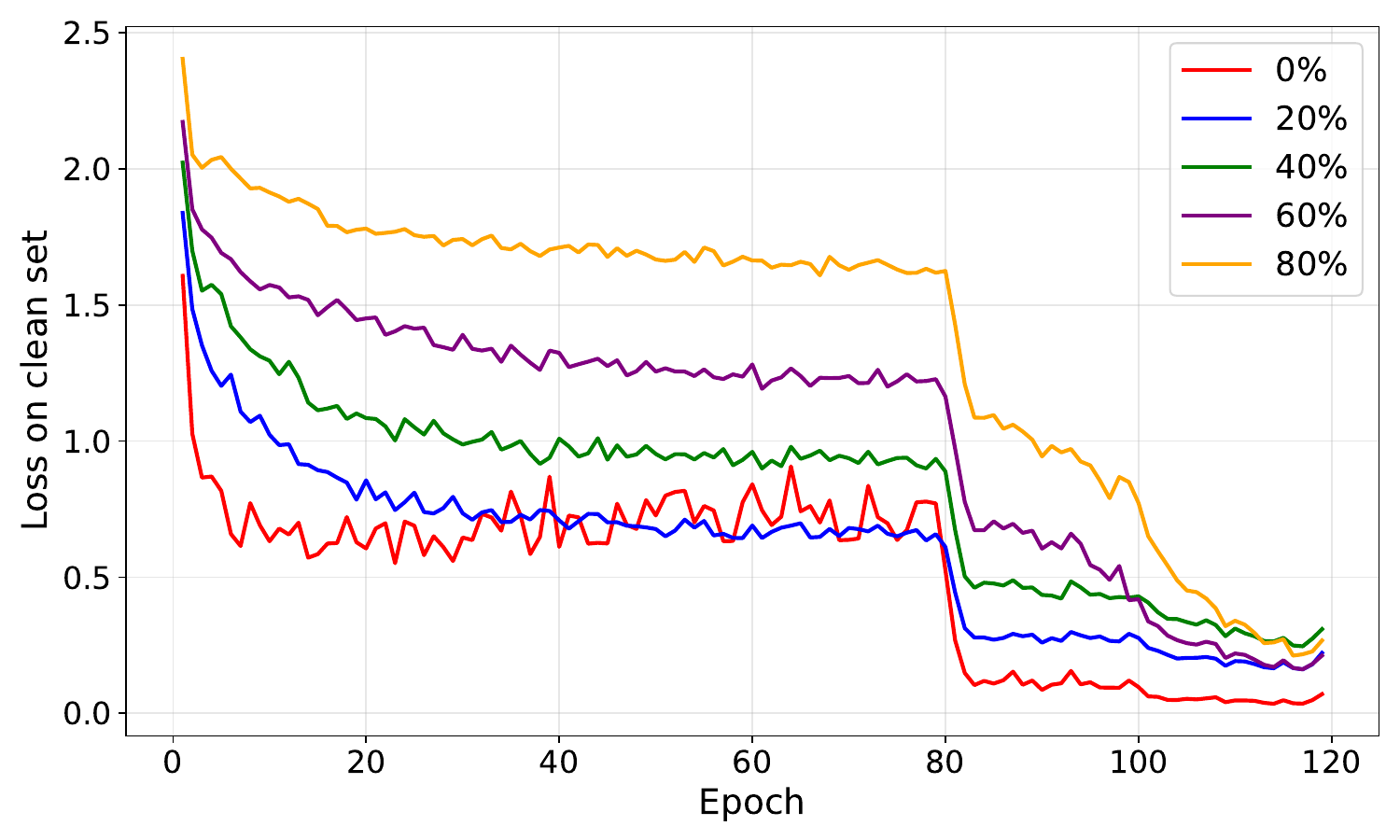}
        \caption{Loss of main model to the clean dataset in CIFAR-10}  
        \label{fig3a}
    \end{subfigure}
    \hfill
    \begin{subfigure}[b]{0.48\textwidth}
        \includegraphics[width=\linewidth]{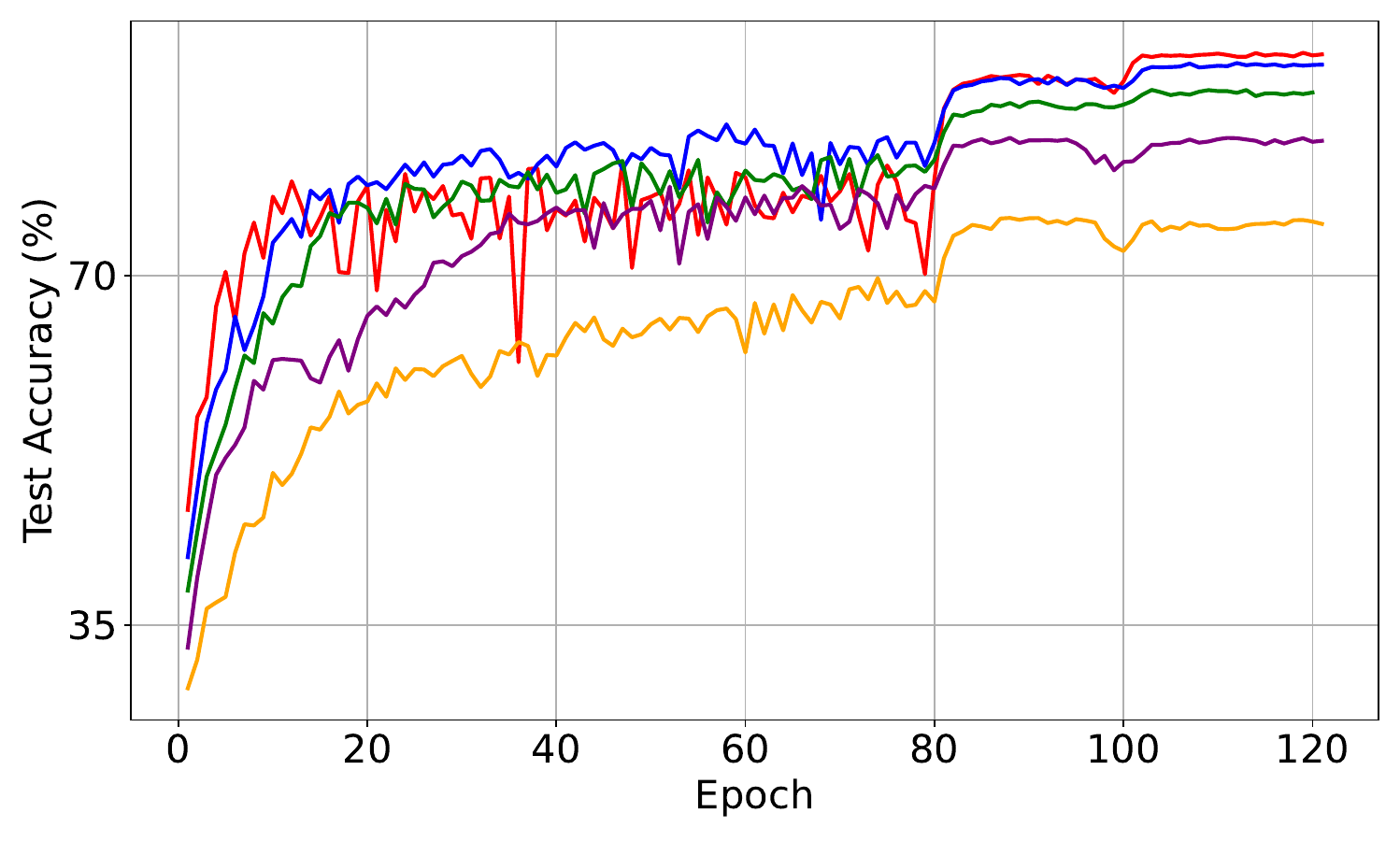}
        \caption{Test accuracy of the main model in CIFAR-10}
    \end{subfigure}
    \hfill
    \begin{subfigure}[b]{0.48\textwidth}
        \includegraphics[width=\linewidth]{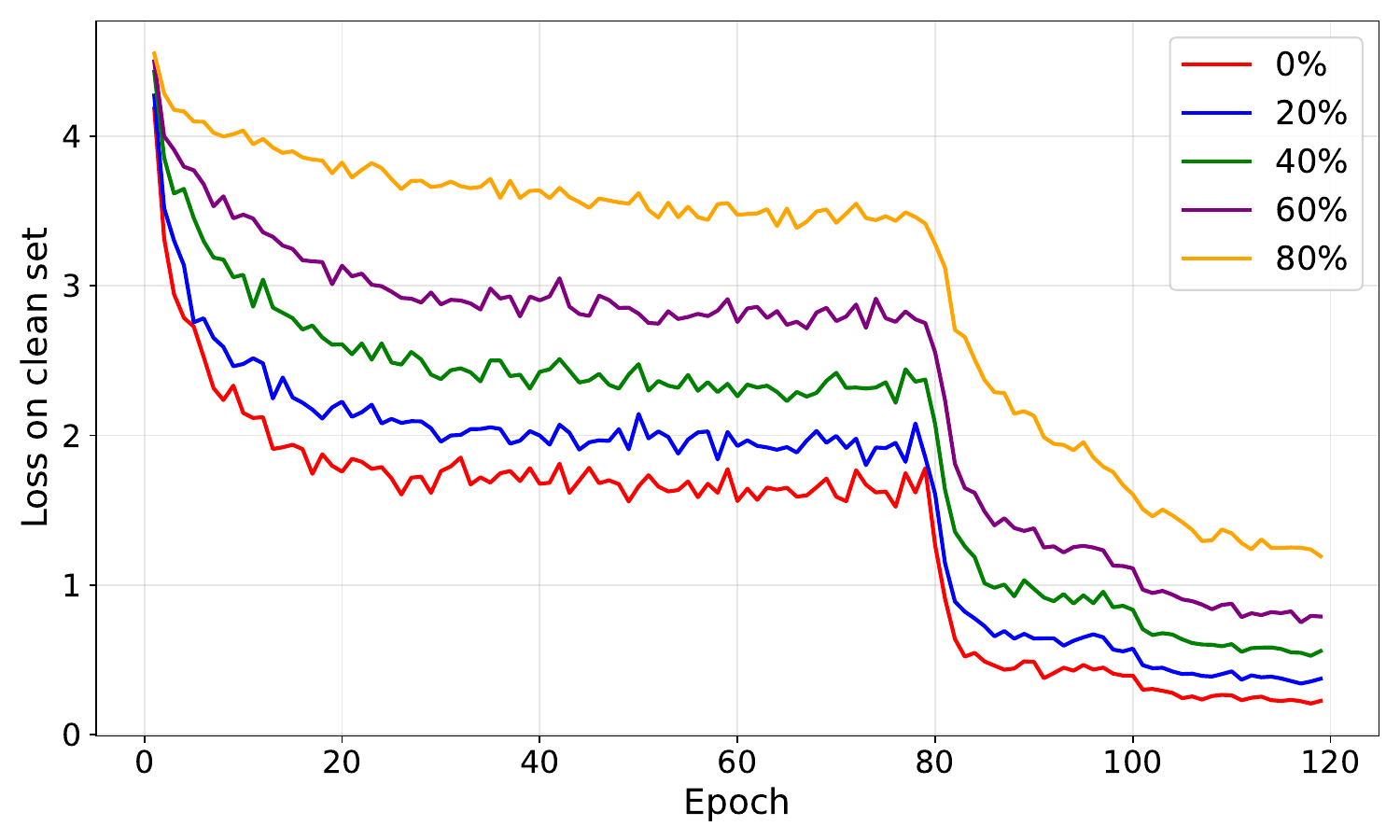}
        \caption{Loss of main model to the clean dataset in CIFAR-100}
        \label{fig3c}
    \end{subfigure}
    \hfill
    \begin{subfigure}[b]{0.48\textwidth}
        \includegraphics[width=\linewidth]{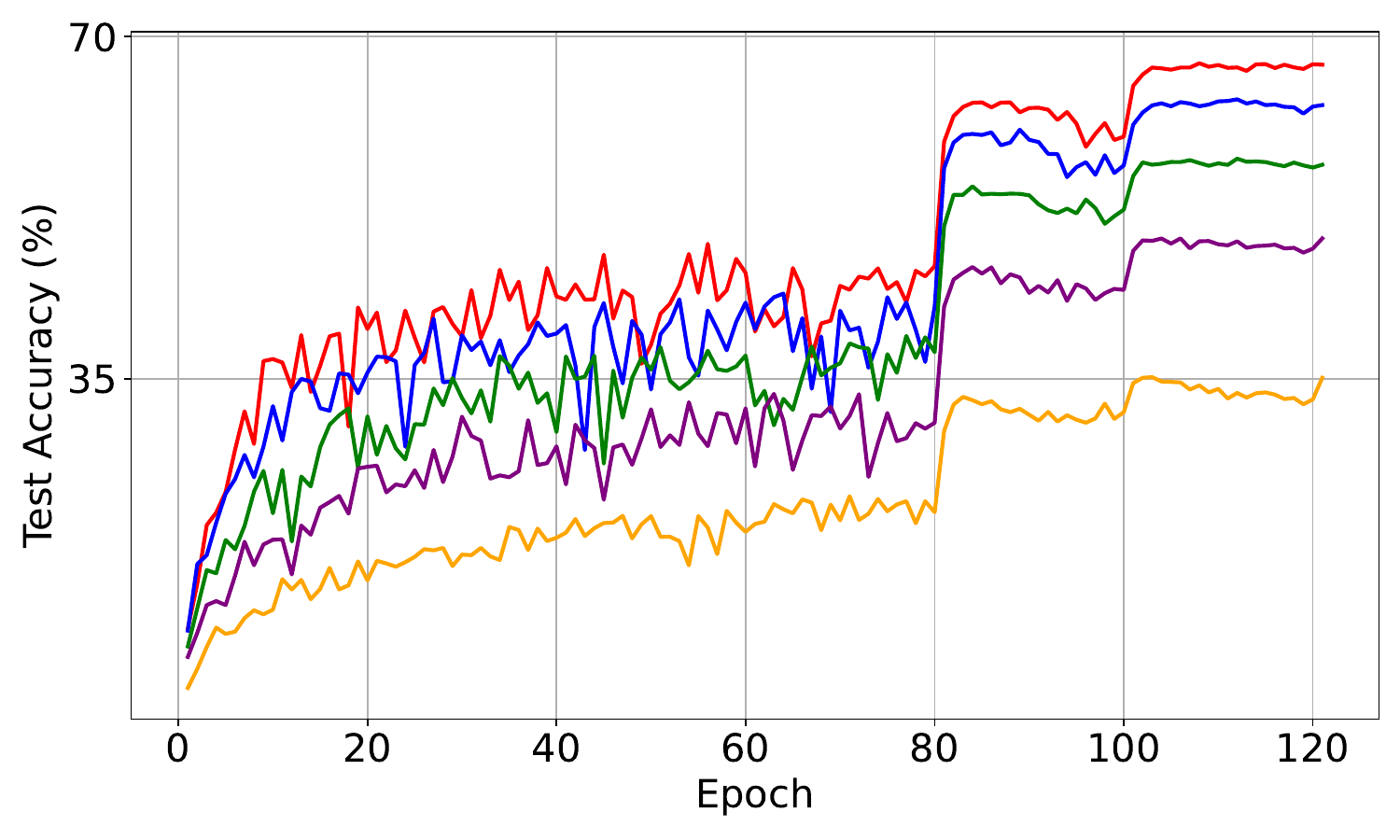}
        \caption{Test accuracy of the main model in CIFAR-100}
    \end{subfigure}
    \caption{Performance of main model on clean dataset and test set under uniform noise ratio on CIFAR-10 and CIFAR-100.}
    \label{losschart}
\end{figure*}

\subsection{Main Results}
The results of our proposed EBOMLC and other baselines are reported in the Table \ref{cifar10} and Table \ref{cifar100} for the CIFAR-10 and CIFAR-100 datasets, respectively.
In Table \ref{cifar10}, our proposed EBOMLC outperfoms all other baselines in most settings.
The GLC method has a slight advantage over EBOMLC in the 20\% flip noise setting.
However, we signifcantly improve accuracy in the high noise region of uniform noise over GLC on CIFAR-10 dataset. This behavior is captured by the Avg. Rank metric.
EBOMLC achieves the lowest Avg. Rank on CIFAR-10.

In Table \ref{cifar100}, similar patterns can be observed on the CIFAR-100 dataset.
Our method outperforms the other baselines, except at the 0\% and 20\% noise rates where GLC is better.
The minor advantage of GLC method over our proposed EBOMLC occurs at low noise rates where its heuristic is able to estimate the class-conditional transition matrix accurately.
As the noise rate increases, the transition matrix becomes harder to denoise, which reduces the effectiveness of GLC's corrections. GLC still achieves the second-best Avg. Rank. However, its rank becomes weaker in the highest-noise settings.
In contrast, EBOMLC does not rely on an explicit noise model, making it more robust when the noise is heavy or less structured, which aligns with the observed gains over GLC in those settings.
For training time comparison, only GLC and Co-teaching+ are more efficient than EBOMLC.
The main reason is that these methods do not depend on the bilevel formulation.
For example, GLC estimates the label noise transition matrix from the clean subset and then trains the main model on the denoised labels directly.
Under the bilevel optimization setting, EBOMLC, which achieves near first-order complexity, significantly improves the training speed of MLC and MLC-D.

To validate the regularization effects of EBOMLC, we further analyse the performance of the main model and its accuracy on the test set in Fig. \ref{losschart}.
We plot the upper loss of the main model $f_w$ to the clean dataset $D$ over training in Fig. \ref{fig3a} and \ref{fig3c}, computed using Eq. \eqref{eq:standardF}.
Note that, our EBOMLC use the proposed $\mathcal{F}$ for training the main model $f_w$ instead of using standard objective function $F$ directly.
As both Theorem \ref{theorem1} and Theorem \ref{theorem2} guarantee the convergence of the proposed upper loss, the empirical results further indicate that the behavior of the main model on the clean dataset follows the same pattern.
By incorporating the distillation mechanism to $\mathcal{F}$, the main model has better regularization properties, which then translates into higher test accuracy;
In the hyperparameter settings, we set the meta model learning rate to be much smaller than that of the main model.
Additionally, the learning rate decay mechanism at epochs 80 and 100 constrains the updates of the meta model $g_\alpha$, resulting in soft label correction.
These choices align with our conditions in Assumption \ref{assumption3}.

Figure \ref{fig:losscompare} compares the proposed EBOMLC against MLC-D to highlight the effectiveness of our method, using the definitions of upper loss in Eq.~\eqref{eq:standardF} and the lower loss in Eq.~\eqref{eq:standardG}.
It is clear that MLC-D heavily prioritizes optimizing the upper loss over the lower loss.
However, due to the small number of samples in the clean dataset $D$, this optimization does not translate into a better test accuracy.
At 40\%{} uniform noise, the lower loss of MLC-D fluctuates substantially, which indicates a low accuracy of the meta model on the noisy dataset $D^{'}$.
This phenomenon matches the plotted heatmap of MLC-D in Fig. \ref{fig:mlcdheatmap}.
In comparison to MLC-D, the main model of EBOMLC maintains a good balance between upper and lower losses throughout training, which then transfer to a better test accuracy.
This is a direct consequence of our regularized objective loss $\mathcal{F}$ which reduces the main model's sensitivity to the limited amount of data in clean subset.
Moreover, the introduced coefficient $\xi$ attenuates the influence of noisy signals, stabilizing the training of main model when learning from both clean and corrupted supervision.
For training speed comparison, MLC-D requires $k=5$ steps to optimize in the inner loop, resulting in a higher computational cost than our EBOMLC.
The training time in Fig. \ref{fig:losscompare} highlights that our EBOMLC is trained at more than twice speed of MLC-D under the same setting, consistent with the reports in Table~\ref{cifar10} and \ref{cifar100}.


Figure~\ref{fig:comparisonheatmap} analyzes the label correction capability of the meta model between the original MLC method and our proposed approach, which is a key factor determining the effectiveness of the label correction mechanism.
In Fig. \eqref{fig5:subfig1}, although the label correction heatmap of MLC is significant better than the introduced MLC-D in Fig.~\ref{fig:mlcdheatmap}, it still struggles to reach at least 60\%{} accuracy under the 40\%{} noise rate in $D'$. 
This weakness stems from the fact that MLC lacks a mechanism enabling the meta model $g_\alpha$ to identify examples whose labels are already correct in noisy set $D'$.
In EBOMLC, by incorporating the gradient $\nabla_\alpha \mathcal{F}$ to learn, our meta model can accurately correct such instances.
Moreover, the proposed dynamic barrier keeps the meta parameters $\alpha$ aligned with the direction of $\nabla_\alpha \mathcal{F}$, thereby improving the meta model's label correction capability.  
This highlights the advantage of EBOMLC over MLC and MLC-D in the high and complicated noise rate settings, as shown more clearly in the Table \ref{cifar100}.


\begin{figure}[H]
    \centering
    \includegraphics[width=0.5\textwidth, height=0.18\textwidth]{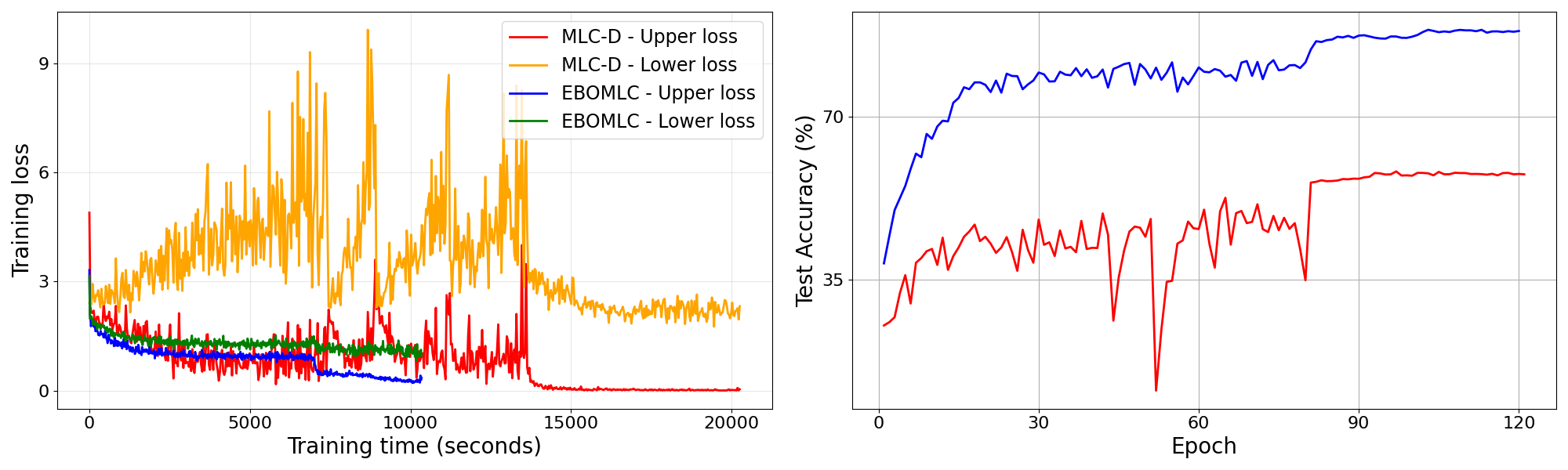}
    \caption{In the left subfigure, the loss curves of EBOMLC and MLC-D on the CIFAR-10 dataset with 40\% uniform noise are plotted. In the right figure, the test accuracy of EBOMLC (blue color) and MLC-D (red color) are plotted.}
    \label{fig:losscompare}
\end{figure}

\subsection{Analysis}
\label{analysis}
\subsubsection{The effect of one step gradient accumulation in updating main model}
\begin{figure}[htbp]
    \centering
    \begin{subfigure}[b]{0.95\columnwidth}
        \centering
        \includegraphics[width=\linewidth]{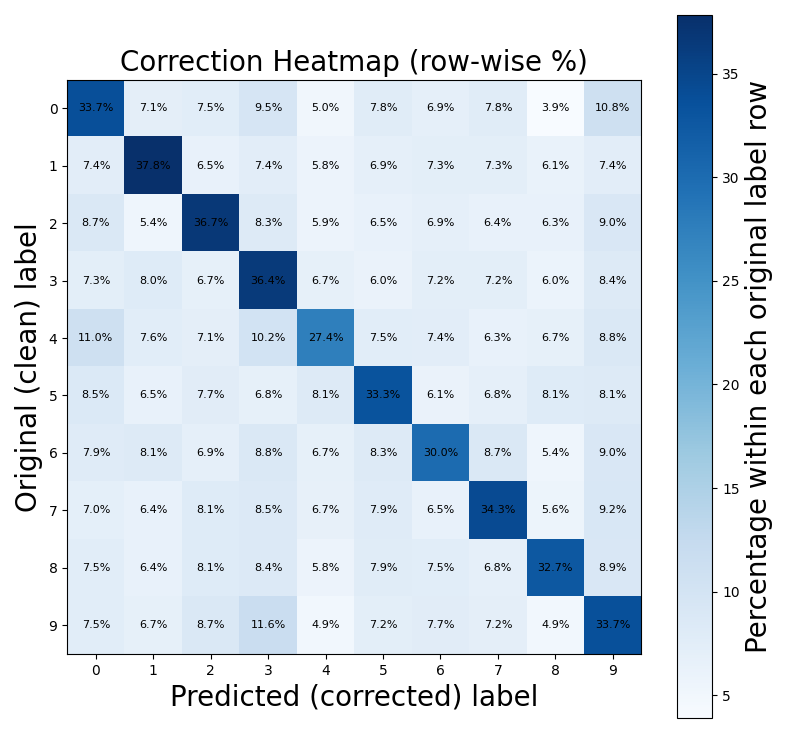}
        \caption{The label correction performance of MLC}
        \label{fig5:subfig1}
    \end{subfigure}
    \hfill
    \begin{subfigure}[b]{0.95\columnwidth}
        \centering
        \includegraphics[width=\linewidth]{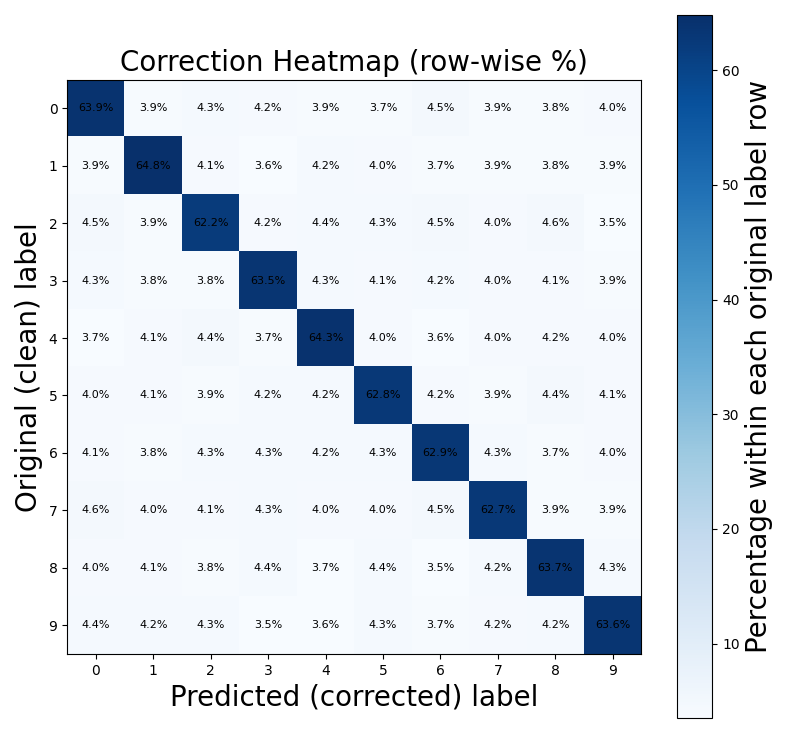}
        \caption{The label correction performance of EBOMLC}
        \label{fig5:subfig2}
    \end{subfigure}
    \caption{Heatmap between predictions of the meta model and the ground-truth labels of the CIFAR-10 training data under 40\% uniform noise after training.}
    \label{fig:comparisonheatmap}
\end{figure}
While MLC-D estimates the optimized $w^*$ using an approximation with $k=5$ inner loop steps, our proposed EBOMLC chooses to use a one-step update from the lower loss instead of computing the optimal parameters $w^*$ in Eq.~\eqref{eq:Qnew}.
The main reason is that a small step of the inner-loop helps prevent noise leakage from noisy labels while preserving the convergence properties of the main parameters $w$.
In this analysis, we empirically show that a one-step inner update (\(k=1\)) is preferable both computationally and statistically.

Given the implementation details in Algorithm \ref{alg2}, it is straightforward to extend EBOMLC to more than one inner loop step.
Figure~\ref{fig:innerloop1} plots the training progress of $k=2$ and $k=5$ compared with $k=1$, and Fig.~\ref{fig:innerloop2} plots the accuracy of the main model on the CIFAR-10 dataset.
We vary the noisy rate in 20\% (low), 40\% (medium), and 80\% (high), using a uniform noise distribution.
All models are trained with 120 epochs, and training time is reported to compare efficiency of across different values of $k$.

The results in Fig.~\ref{fig:innerloop} highlight that EBOMLC with $k=1$ demonstrates the smoothest training dynamics and achieves the highest test accuracy.
When the number of steps $k$ increases, the loss becomes more unstable and test performance decreases.
This aligns with our above analysis that a large $\nabla \mathcal{Q}$ magnitude would introduce more noisy signal into the main model.
At a 80\%{} noise rate, both $k=2$ and $k=5$ have a significant drop in test accuracy compared with the standard setting $k=1$.
Moreover, the default value $k=1$ achieves the best training efficiency, which is $1.5\times$ and $2\times$ faster than $k=2$ and $k=5$, respectively.

\begin{figure}[htbp]
    \centering
    \begin{subfigure}[b]{0.9\columnwidth}
        \centering
        \includegraphics[width=\linewidth]{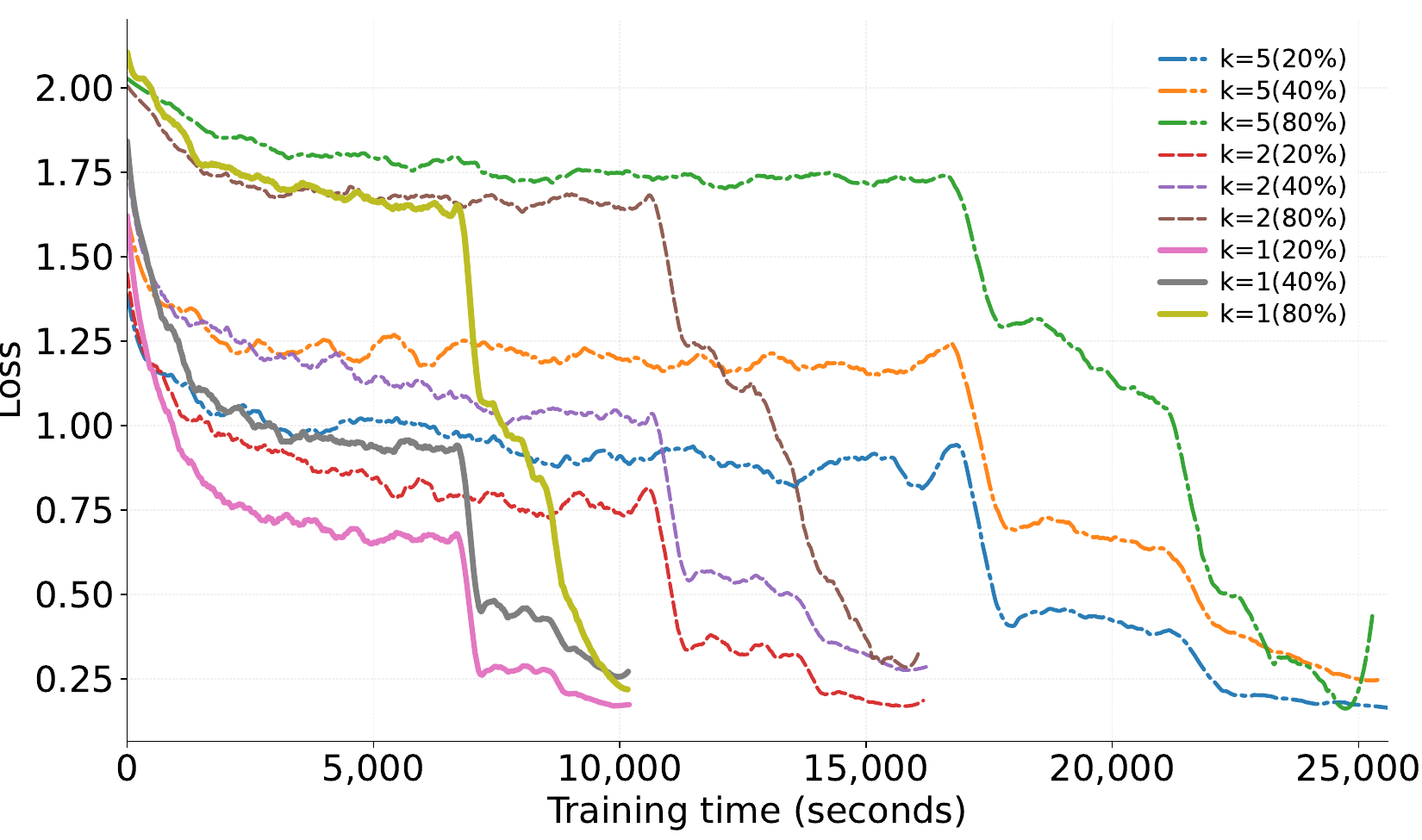}
        \caption{Loss on clean set}
        \label{fig:innerloop1}
    \end{subfigure}
    \hfill
    \begin{subfigure}[b]{0.9\columnwidth}
        \centering
        \includegraphics[width=\linewidth]{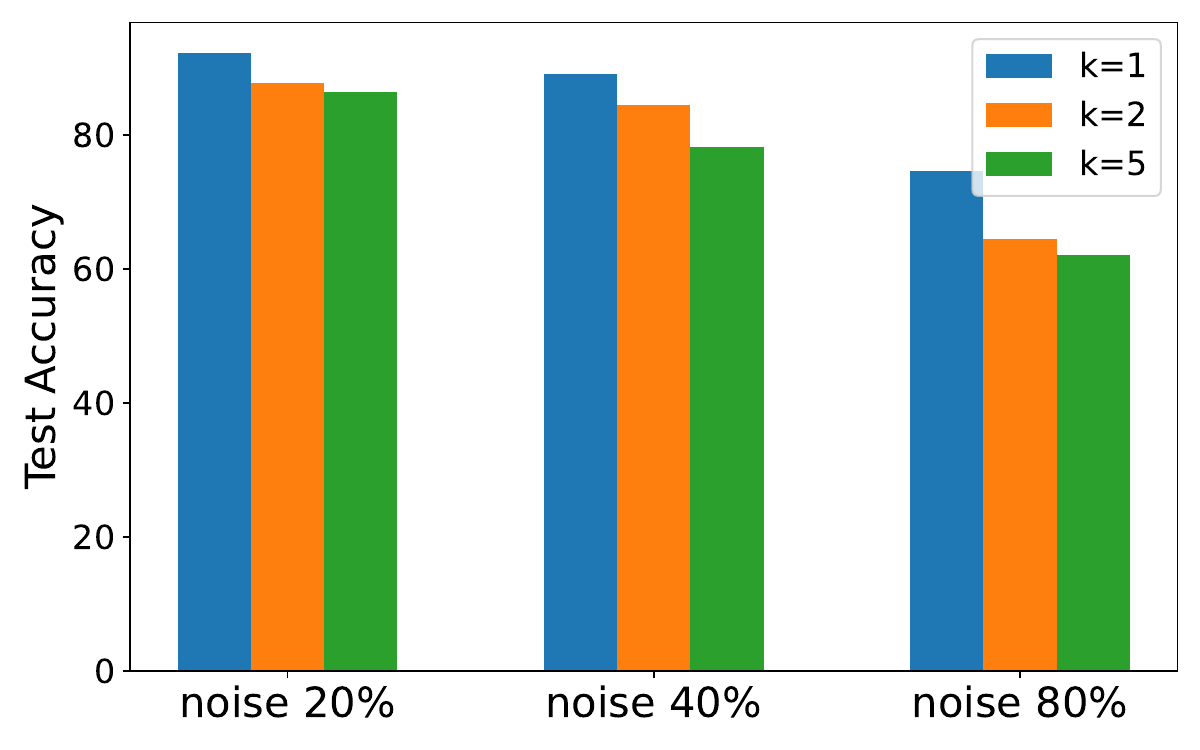}
        \caption{The accuracy on test set}
        \label{fig:innerloop2}
    \end{subfigure}
    \caption{Comparison of EBOMLC using different values of $k$ in $20\%{}$, $40\%{}$ and $40\%{}$ uniform noise.}
    \label{fig:innerloop}
\end{figure}

\begin{figure}[http]
    \centering
    \includegraphics[width=0.5\textwidth, height=0.18\textwidth]{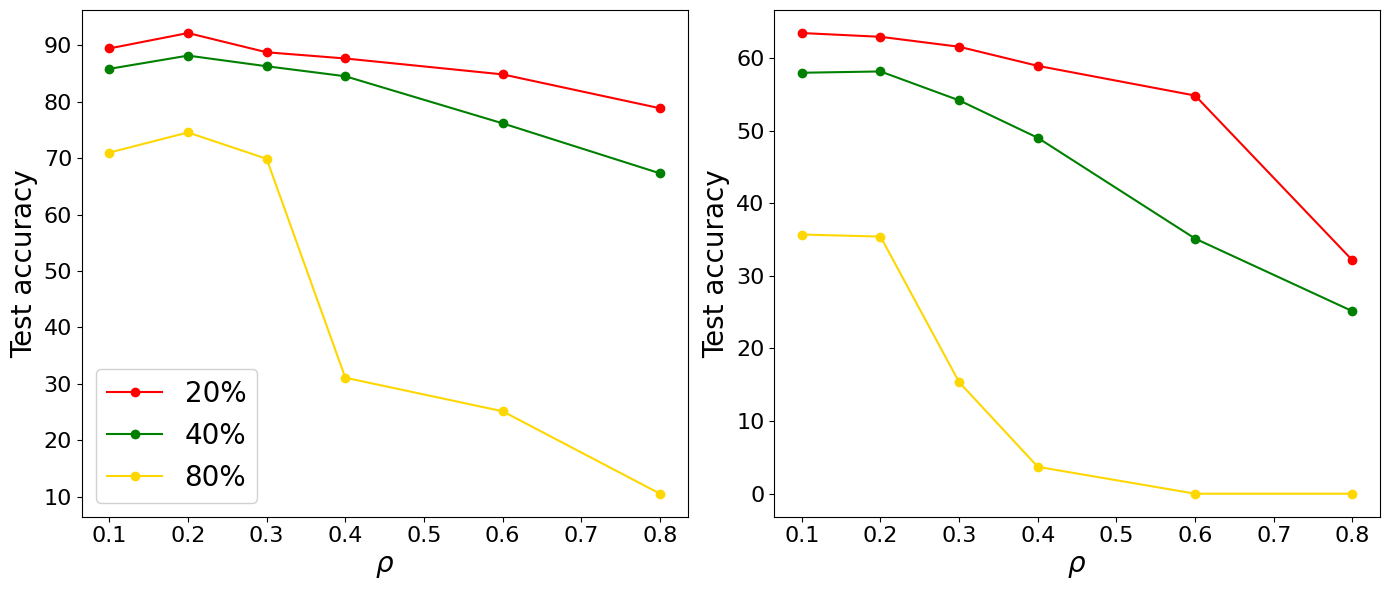}
    \caption{Test accuracy w.r.t $\rho$ values on CIFAR-10 (left) and CIFAR-100 (right).}
    \label{fig:rho}
\end{figure}


\subsubsection{Analysis of the mixture upper loss $\mathcal{F}$}
The proposed objective function $\mathcal{F}$ in EBOMLC prevents the main model from overfitting on the clean dataset by relying on information from the meta model.
The degree of knowledge distillation is controlled by the hyperparameter $\rho$.
In Fig. \ref{fig:rho}, we plot the test accuracy for different value of $\rho$ in the range of $[0.1, 0.8]$ on both CIFAR-10 and CIFAR-100 dataset.
The noise rate is selected from 20\%{}, 40\%{}, and 80\%{} with uniform setting.

We observe that EBOMLC achieves the best performance on test accuracy with a low value of $\rho$.
When the value of $\rho$ increases, the objective function $\mathcal{F}$ becomes dominated by the loss on the clean dataset.
From a theoretical aspect, both Theorem \ref{theorem1} and \ref{theorem2} still hold for any value of $\rho \in (0,1)$, implying that EBOMLC is able to converge on the small, clean training dataset to some extent.
However, it starts to overfit the clean samples with large values of $\rho$, which is similar to the behavior of MLC-D.
At an 80\%{} noise rate, the regularization effect of $\mathcal{F}$ becomes critical.
The performance of EBOMLC with $\rho=0.8$ even collapses to a very low accuracy on the test set.
Meanwhile, with $\rho=0.1$ or $\rho=0.2$, EBOMLC still maintains stable performance on both testing datasets.

Based on our experiments, a setting of $\rho = 0.2$ performs robustly across noise types and rates.
This suggests that the optimal $\rho$ is governed by the ratio of clean to noisy samples in the training set rather than an engineering choice.
In this work, we restrict our experiments to a fixed ratio to focus on comparing EBOMLC with prior methods.
A more general approach is to search for the optimal value of the hyperparameter by deriving a principled estimator for $\rho$ given the clean and noisy training dataset.
\begin{figure}[H]
    \centering
    \includegraphics[width=0.5\textwidth, height=0.18\textwidth]{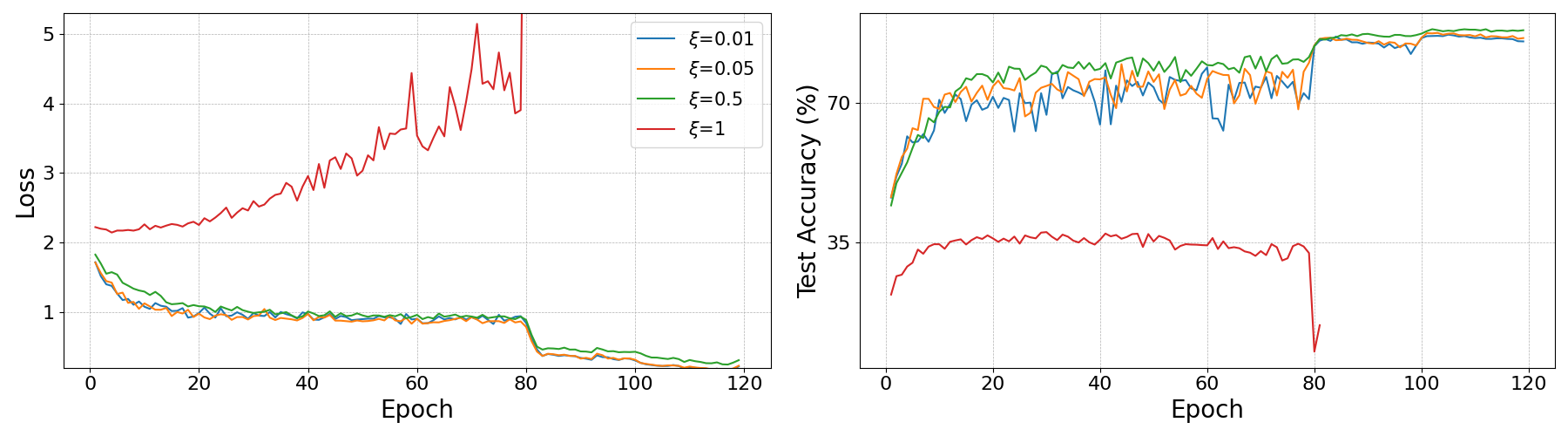}
    \caption{In the left subfigure, the loss on clean dataset w.r.t $\xi$ with 40\% uniform noise are plotted. In the right figure, the test accuracy of EBOMLC with different choise of $\xi$ are plotted.}
    \label{fig:mchart}
\end{figure}

\subsubsection{Analysis of the training dynamics with the scaling constant $\xi$}
In EBOMLC, the scaling constant $\xi$ is introduced to scale the noisy gradient $\|\nabla \mathcal{Q}\|$ in the update direction.
Figure \ref{fig:mchart} illustrates the training dynamics of EBOMLC with $\xi \in \{0.01, 0.05, 0.5, 1\}$ when training on CIFAR-10 dataset with 40\% uniform noise.
When $\xi$ is set to 1, the main model fails to converge with respect to the loss on the clean set after epoch 80 and the test accuracy starts to drop.
This issue arises when the update direction for $w$ in Eq. \eqref{eq:ourupdatew} is dominated by $\bar\beta_t \nabla_w \mathcal{Q}$ instead of $\nabla_w \mathcal{F}$.
As training proceeds, the magnitude of $\| \nabla_w \mathcal{Q}\|$ shrinks toward zero while $\bar\beta_t $ increases, creating a positive-feedback loop. 
As a result, the noisy term amplifies and leaks noise into the main model, leading to unstable training dynamics.
At the other extreme, when $\xi$ is too small, $\nabla_w$ aligns mostly with $\nabla_w \mathcal{F}$.
Thus, it results in tighter convergence of the loss on the clean set but under-utilizes the noisy set, which could potentially reduce generalization to test data.
In the right plot of Fig.~\ref{fig:mchart}, the test accuracy of EBOMLC becomes more fluctuation when $\xi$ decreases.
Therefore, choosing an appropriate value of $\xi$ is essential in practice.
Based on our experiments, a reasonable range of $\xi$ is $[0.1, 0.5]$.
Within this interval, the main model achieves consistent learning dynamics and has robust performance across different noise rates and settings.
\subsection{Limitations}
{Although EBOMLC shows strong results, several limitations remain.
First, EBOMLC leverages a clean subset to define the upper-level objective, which makes the meta signal reliable but also introduces an extra requirement.
In many real scenarios, such trusted data may be unavailable or expensive to collect, and even when it exists, the subset can be very small.
A promising direction is to reduce this dependence by constructing a trusted set automatically from noisy data.
Another direction is to couple EBOMLC with semi-supervised learning, so unlabeled or weakly-labeled samples can strengthen the upper signal when clean labels are scarce.}

{Second, EBOMLC still requires several key hyperparameters, and training quality depends on them.
The scaling constant $\xi$ affects stability of the bilevel update. Beside that, EBOMLC introduces the mixture weight $\rho$ also matters because it may shift across datasets.
In the current setting, these values are tuned manually, which can reduce practicality when conditions change.
Future work can design adaptive schedules for $\xi$ and $\rho$ based on training dynamics, such as validation trends or gradient statistics.}

\section{Conclusion}
In this paper, we introduce EBOMLC, a novel and efficient framework for learning with noisy labels.
By applying the dynamic barrier gradient descent algorithm, we improve the efficiency of the MLC framework, achieving the first-order complexity approximately.
However, applying this scheme naively makes the training dynamic become unstable and more sensitive to noisy signals.
Therefore, our proposed EBOMLC incorporates three key improvements, including a one-step inner loop update, a mixture upper loss, and an alignment-aware dynamic barrier.
EBOMLC significantly reduces computational overhead and stabilizes the training process while enhancing label correction capabilities.
Experiments on CIFAR-10 and CIFAR-100 demonstrate that our method outperforms existing baselines, especially under high-noise conditions, both in accuracy and efficiency.
Given the promise of dynamic barrier gradient descent in noisy label learning, we aim to extend EBOMLC both theoretical and experimental directions.
One possible direction is to analyze the effects of a large learning rate in the meta model to the main model, which is not considered in our setting.
Additionally, the design of the mixture upper loss can be extended with further regularization techniques to make our EBOMLC capable of handling more general settings with varying noise distributions between the clean and noisy sets.
\input{appendix}

\vspace{12pt}
\bibliographystyle{unsrtnat}
\bibliography{references}

\end{document}

%% file: relatedwork.tex

\section{Related Works}
From an optimization strategy perspective, current methods for learning with noisy labels are roughly divided into two categories: sample selection and label correction.
This taxonomy is commonly adopted in recent surveys and benchmark reviews of learning with noisy labels \cite{song2022,algan2021survey}.

Sample selection approaches \cite{jiang2018, han2018} focus on detecting noisy samples and reducing their influence during training.
Several strategies for the selection mechanism are proposed such as co-teaching with multiple network \cite{han2018, yu2019}, curriculum and disagreement-based training \cite{jiang2018mentornet,malach2017decoupling}, or weight-based approaches \cite{ren2018, shu2019}.
{DivideMix \cite{dividemix2020} combines loss-based splitting with semi-supervised learning (SSL) to exploit the uncertain subset}.
SplitNet \cite{kim2025splitnet} extends this direction by making the splitting step learnable. {However, it adds an extra learnable splitting network, increasing training complexity, so improvements can be partly due to the added module rather than the core framework.}

In the branch of label correction, classical approaches rely on estimating the noise distribution by a confusion matrix between noisy labels and correct labels \cite{hendrycks2018, patrini2017}.
The label denoising process can also be derived from the model’s prediction output through a self-correction mechanism \cite{wang2021}, or by mixing observed labels with model predictions (bootstrapping) \cite{reed2015bootstrapping}.
Robust objectives and regularization are also commonly used to reduce memorization under noise, e.g., robust loss designs \cite{zhang2018gce,wang2019symmetric} and early-learning regularization \cite{liu2020elr}.
While these methods achieve stable performance in the presence of noise, they are mostly designed based on metaheuristics.
For example, estimating the noise-transition matrix requires strong assumptions about the underlying label noise distribution.
In the GLC approach \cite{hendrycks2018}, the author employs a conditional independent assumption between the true label and the noisy label, given the input sample, to learn the noise correction matrix.
ProSelfLC \cite{wang2021} assumes a connection between the model’s prediction entropy and the predicted labels in noisy data.
{DULC \cite{xu2025dulc} introduces an interpolation-based label correction to avoid false correction.}
However, these approaches are more sensitive to noise in the early training phase, especially in a high noise rate setting.

The MLC framework is proposed in \cite{zheng2021}, which links the main model training process to an auxiliary label-correction model via a bilevel optimization setup.
MLC then updates the label correction model using a $k$-step look-ahead stochastic gradient descent (SGD).
While MLC outperforms other baselines in term of accuracy, it still incurs high computational cost due to the estimation of Hessian-vector products on the noisy data.
{Recently, the meta label correction line has been strengthened by revisiting meta-gradient approximations and teacher/student designs \cite{taraday2023emlc}.
DMLP \cite{tu2023dmlp} decouples the representation learning procedure from label correction via multi-stage pipelines.
However, most of these recent improvements often requires stronger representation learning recipes, which alternates the training protocol and computational budget.
In our work, we keep the backbone and recipe aligned with the standard MLC setting and focus on the bilevel optimization of meta label framework.}

Similar to MLC, other weighting-based approaches built upon meta learning principles are proposed in \cite{ren2018, shu2019}.
L2RW \cite{ren2018} automatically assigns weights to each sample such that the main model achieves a low loss on the validation set.
MW-Net \cite{shu2019} replaces the linear weighting scheme with a multilayer perceptron.
{Later, PMW-Net \cite{zhao2021} proposes a probabilistic formulation of MW-Net that improves the performance of standard MW-Net.
Although these methods could potentially discard valuable information carried by noisy samples, they typically employ cheaper bilevel approximations than the multi-step bilevel scheme used in MLC, and thus can be more efficient in practice.}

%% file: appendix.tex
\begin{appendix}

\setcounter{lemma}{0}
In the section, we first provide key lemmas to work with the proofs of Theorem 1 and Theorem 2 in the main text in Appendix Section \ref{lemmas}. The proof of Theorem \ref{theorem1} and Theorem~\ref{theorem2} are presented in Section \ref{proof1} and Section \ref{proof2}, respectively.
\subsection{Lemmas}
\label{lemmas}


\begin{lemma}\label{lemma1}
If $\mathcal{F}$ are differentiable L-Lipschitz functions w.r.t the joint $(w,\alpha)$, then
\setlength{\arraycolsep}{0.0em}
\begin{eqnarray}
    \mathcal{F}(w_1,\alpha) - \mathcal{F}(w_2,\alpha) &{}\leq{}&\langle \nabla_w \mathcal{F}(w,\alpha) \big|_{w=w_2},w_1-w_2 \rangle \nonumber\\
    &&{+} \dfrac{L}{2} \|w_1 - w_2 \|^2_2 \label{eq:lemma1w}
\end{eqnarray}
\begin{eqnarray}
    \mathcal{F}(w,\alpha_1) - \mathcal{F}(w,\alpha_2) &{}\leq{}& \langle \nabla_\alpha \mathcal{F}(w,\alpha) \big|_{\alpha = \alpha_2}, \alpha_1 - \alpha_2 \rangle \nonumber \\
    &&{+} \dfrac{L}{2} \| \alpha_1 -\alpha_2 \|^2_2 \label{eq:lemma1a}
\end{eqnarray}
\setlength{\arraycolsep}{5pt}

\end{lemma}
\begin{IEEEproof}
It is straight forward to derive Lemma \ref{lemma1} based on Assumption \ref{assumption1}.
\end{IEEEproof}
\begin{lemma}\label{lemma2}
 Under Assumption \ref{assumption1} and \ref{assumption2}, we have \[\bar{\beta}_t \| \nabla \mathcal{Q}(w^{(t)},\alpha^{(t)}) \| \leq  M(2\delta  +1) \]lead to $\bar{\beta}_t \| \nabla_w \mathcal{Q}(w^{(t)},\alpha^{(t)}) \|$ and $\bar{\beta}_t \| \nabla_\alpha \mathcal{Q}(w^{(t)},\alpha^{(t)}) \|$ are upper bounded by $M(2\delta  +1)$.\\
\end{lemma}
\begin{IEEEproof} By the definition of $\bar{\beta}_t$ in Eq. \eqref{newbeta}, we obtain: 
\begin{flalign*}
     \bar{\beta}_t &\leq \delta + \dfrac{\| \nabla_w \mathcal{F}(w^t,\alpha^t)\| \|\nabla_w \mathcal{Q}(w^t,\alpha^t) \|}{\| \nabla \mathcal{Q}(w^t,\alpha^t)\|^2}  \\
     & \leq \delta + \dfrac{\| \nabla_w \mathcal{F}(w^t,\alpha^t)\| }{\| \nabla \mathcal{Q}(w^t,\alpha^t)\|} 
\end{flalign*}
Therefore, the l.h.s of the Lemma \ref{lemma2} is bounded by
\begin{flalign*}
\bar{\beta}_t \|\nabla \mathcal{Q}(w^t,\alpha^t)\|  &\leq (\delta \| \nabla \mathcal{Q}(w^t,\alpha^t) \| + \|\nabla_w \mathcal{F}(w^t,\alpha^t)\|) \\&\leq (2\delta M + M) = M(2\delta+1)
\end{flalign*}
The last equality holds because \[\| \nabla \mathcal{Q}(w^t,\alpha^t)\| = \| \nabla G(w^t,\alpha^t) - \nabla G(w^{(t)}_{(1)},\alpha^t)\| \leq 2M\]
\end{IEEEproof}
\begin{lemma}\label{lemma3}
    Let $(a_n)_{n \leq 1}, (b_n)_{n \leq 1}$ be two non-negative real sequences such that the series $\sum_{i=1}^{\infty} a_n$ diverges, the series $\sum_{i=1}^{\infty} a_n b_n$ converges, and there exists $K > 0$ such that  
\[
|b_{n+1} - b_n| \leq K a_n.
\]
Then, the sequence $(b_n)_{n \leq 1}$ converges to $0$. \\
\end{lemma}
\begin{IEEEproof}
 Since $\sum_{i=1}^{\infty} a_n$ diverges and $\sum_{i=1}^{\infty} a_nb_n < \infty$, we necessarily have $\liminf_{n\to\infty} b_n=0$

Assume for the contradiction that $\limsup_{n\to\infty} b_n=\lambda>0$.
Choose increasing index sequences $(m_j)_{j\ge1}$ and $(n_j)_{j\ge1}$ such that
\begin{flalign*}
m_j&<n_j<m_{j+1},\\
\frac{\lambda}{3}&< b_k \ \text{for } m_j\le k<n_j, \\
b_k&\le \frac{\lambda}{3} \ \text{for } n_j\le k<m_{j+1}.
\end{flalign*}
Let $\varepsilon=\lambda^{2}/(9K)$ and pick $\tilde{\jmath}$ so large that
\[
\sum_{n=m_{\tilde{\jmath}}}^{\infty} a_n b_n<\varepsilon.
\]
Then, for every $j\ge\tilde{\jmath}$ and every $m$ with $m_j\le m\le n_j-1$,
\begin{align*}
|b_{n_j}-b_m|
&\le \sum_{k=m}^{n_j-1} |b_{k+1}-b_k|
 \le \sum_{k=m}^{n_j-1} K a_k
     = \frac{3K}{\lambda}\sum_{k=m}^{n_j-1} a_k \frac{\lambda}{3} \\
&\le \frac{3K}{\lambda}\sum_{k=m}^{n_j-1} a_k b_k
 \le \frac{3K}{\lambda}\sum_{k=m}^{\infty} a_k b_k
 \le \frac{3K}{\lambda}\,\varepsilon
 = \frac{\lambda}{3}.
\end{align*}
Using the triangle inequality and the construction $b_{n_j}\le \lambda/3$, we obtain
\[
b_m \le b_{n_j} + |b_{n_j}-b_m| \le \frac{\lambda}{3} + \frac{\lambda}{3}
= \frac{2\lambda}{3}.
\]
Hence $b_m\le 2\lambda/3$ for all $m$ large enough, which contradicts
$\limsup_{n\to\infty} b_n=\lambda>0$. Therefore, $b_n\to 0$.
\end{IEEEproof}

\subsection{Proof of Theorem 1}
\label{proof1}
\begin{IEEEproof}
We first decompose the one step update of $\mathcal{F}(w^t, \alpha^t)$ as follows:
\begin{align}
\mathcal{F}(w^{t+1},\alpha^{t+1}) - \mathcal{F}(w^t,\alpha^t) &= \mathcal{F}(w^{t+1}, \alpha^{t+1}) - \mathcal{F}(w^{t+1}, \alpha^{t}) \nonumber \\
&+ \mathcal{F}(w^{t+1}, \alpha^t) -\mathcal{F}(w^t, \alpha^t) \label{eq:theoremdecomposeF}
\end{align}
Expand the first term in Eq. \eqref{eq:theoremdecomposeF}, using Eq. \eqref{eq:lemma1a}:
\begin{align}
    \mathcal{F}(&w^{t+1}, \alpha^{t+1}) - \mathcal{F}(w^{t+1}, \alpha^t) \leq \nonumber \\
    & \nabla_\alpha \mathcal{F}(w^{t+1}, \alpha^t)(\alpha^{t+1} - \alpha^t) + \dfrac{L}{2}\|\alpha^{t+1} - \alpha^t\|^2_2 \label{eq:theorem1firstterm}
\end{align}
By definition of the updated meta parameters, we have
\begin{equation}
\alpha^{t+1} - \alpha^t = -\eta_\alpha (\nabla_\alpha \mathcal{F}(w^t, \alpha^t) + \xi \bar{\beta}_t \nabla_\alpha \mathcal{Q}(w^t, \alpha^t)
\label{appendix:eqalpha}
\end{equation}
Since the gradient $|\nabla_\alpha \mathcal{F}(w, \alpha)|$ is bounded by $M$ and $|\bar{\beta}_t \nabla_\alpha \mathcal{Q}(w^t, \alpha^t)|$ is bounded by $M(2\delta  +1)$  (Lemma \ref{lemma2}), the update of $\alpha$ is bounded by:
\begin{equation}
\alpha^{t+1} - \alpha^t \leq \eta_\alpha( M+M(2 \delta  +1))
\label{eq:alphastep}
\end{equation}
Therefore, from Eq. \ref{eq:theorem1firstterm} and Eq. \ref{appendix:eqalpha}, we can bound the update step of $\mathcal{F}$ along the $\alpha$ direction by
\begin{flalign}
    & \mathcal{F}(w^{t+1}, \alpha^{t+1}) - \mathcal{F}(w^{t+1}, \alpha^t) && \nonumber \\
    & \leq \eta_\alpha( M^2+M^2(2 \delta  +1)) + (\eta_\alpha)^2 \mathcal{K}_2 \nonumber \\
    & \leq \eta_\alpha \mathcal{K}_1 + (\eta_\alpha)^2 \mathcal{K}_2 \label{eq:theorem1a}
\end{flalign}

Where $\mathcal{K}_1 =M^2+M^2(2 \delta  +1)$ and $\mathcal{K}_2 = \dfrac{L}{2}(M + M(2\delta +1))^2$ as constants. Similarly, using Eq. \eqref{eq:lemma1w} and the update rule of $w$ we can bound the second term of Eq. \eqref{eq:theoremdecomposeF} as
\begin{flalign}
    & \mathcal{F}(w^{t+1}, \alpha^t) - \mathcal{F}(w^t,\alpha^t) && \nonumber \\
    & \leq \nabla_w \mathcal{F}(w^t,\alpha^t)(w^{t+1}-w^t) + \dfrac{L}{2} \| w^{t+1} -w^t \|^2 \label{eq:t1secondterm}
\end{flalign}
Substitute $$w^{t+1}-w^t = -\eta_w (\nabla_w \mathcal{F}(w^t,\alpha^t) + \xi \bar{\beta}_t \nabla_w\mathcal{Q}(w^t,\alpha^t)) $$
into the r.h.s of Eq. \eqref{eq:t1secondterm}:
\begin{flalign*}
    & \mathcal{F}(w^{t+1}, \alpha^t) - \mathcal{F}(w^t,\alpha^t) && \nonumber \\
    & \leq \nabla_w \mathcal{F}(w^t,\alpha^t)(-\eta_w (\nabla_w \mathcal{F}(w^t,\alpha^t) + \xi \bar{\beta}_t \nabla_w\mathcal{Q}(w^t,\alpha^t))) &&\\
    & \hspace{0.5cm} + \dfrac{L}{2}(\eta_w)^2 \|\nabla_w \mathcal{F}(w^t,\alpha^t) + \xi \bar{\beta}_t \nabla_w\mathcal{Q}(w^t,\alpha^t)||^2 &&\\
\end{flalign*}
By expanding and rearranging, we get a bound which depends the following three terms
\begin{flalign}
    & \mathcal{F}(w^{t+1}, \alpha^t) - \mathcal{F}(w^t,\alpha^t) \nonumber &&\\
    & \leq (-\eta_w  + \dfrac{L}{2 }(\eta_w)^2 ) \|\nabla_w \mathcal{F}(w^t,\alpha^t)\|^2 \nonumber &&\\
    &\hspace{0.5cm}+ \dfrac{L}{2}(\eta_w)^2 \bar{\beta}_t^2 \| \nabla_w \mathcal{Q}(w^t,\alpha^t)\|^2 && \nonumber \\    
    & \hspace{0.5cm} + (L(\eta_w)^2 + \eta_w \xi) \bar{\beta}_t\|\nabla_w\mathcal{F}(w^t,\alpha^t)\| \|\nabla_w \mathcal{Q}(w^t,\alpha^t)\| \label{eq23:secondterm} 
\end{flalign}
By Assumption \ref{assumption1} and Lemma \ref{lemma2}, we can bound $\|\nabla_w \mathcal{F}(w^t,\alpha^t)\|$ and $\bar{\beta}_t \| \nabla_w \mathcal{Q}(w^t,\alpha^t)\|$.
To bound the third term, we rely on the definition of $\bar{\beta}_t$ as follows:
\begin{flalign*}
    & \bar{\beta}_t \| \nabla_w \mathcal{F}(w^t,\alpha^t)\| \|\nabla_w \mathcal{Q}(w^t,\alpha^t)\| &&\\
    &\leq (\delta + \dfrac{\| \nabla_w \mathcal{F}(w^t,\alpha^t) \|}{\| \nabla_w \mathcal{Q}(w^t,\alpha^t )\|})\| \nabla_w \mathcal{F}(w^t,\alpha^t)\| \|\nabla_w \mathcal{Q}(w^t,\alpha^t)\| &&\\
    &\leq (\delta \| \nabla_w \mathcal{F}(w^t,\alpha^t)\| \|\nabla_w \mathcal{Q}(w^t,\alpha^t)\| + \| \nabla_w \mathcal{F}(w^t,\alpha^t)\|^2) &&\\
    &\leq \| \nabla_w \mathcal{F}(w^t,\alpha^t)\|^2  + \delta M\|\nabla_w \mathcal{Q}(w^t,\alpha^t)\| \hspace{3cm} &&\\
    & = \| \nabla_w \mathcal{F}(w^t,\alpha^t)\|^2  + \delta M\|\nabla_w G(w^t,\alpha^t) - \nabla_w G(w^{*(t)},\alpha^t)\| \refstepcounter{equation}\tag{\theequation} \label{thirdterm}
\end{flalign*}
Under Assumption \ref{assumption1} and \ref{assumption2}, we have
\begin{flalign*}
    & \|\nabla_w G(w^t,\alpha^t) - \nabla_w G(w^{*(t)},\alpha^t)\| &&\\
    & \leq L\|w^t-w^{*(t)}\| = L\|\eta_w \nabla_w G(w^t,\alpha^t)\| \leq \eta_w LM \refstepcounter{equation}\tag{\theequation} \label{inenormG}
\end{flalign*}
Combining the inequality of \eqref{eq23:secondterm}, \eqref{thirdterm} and \eqref{inenormG}, the one step update of the main parameters $w$ is bounded by
\begin{flalign}
    & \mathcal{F}(w^{t+1}, \alpha^t) - \mathcal{F}(w^t,\alpha^t) \nonumber &&\\
    &\leq -((1-\xi)\eta_w - \dfrac{L}{2}(\eta_w)^2) \| \nabla_w \mathcal{F}(w^t,\alpha^t)\|^2 + (\eta_w)^2 \mathcal{K}_3 \label{eq:theorem1b}
\end{flalign}
where $\mathcal{K}_3 = \xi \delta M^2L+ LM^2(2\delta+1)+\dfrac{L}{2}(M(2\delta+1))^2$ as a constant.
Using the above \eqref{eq:theorem1a} and \eqref{eq:theorem1b}, the one step updates of \eqref{eq:theoremdecomposeF} is bounded by
\begin{flalign}
    & \mathcal{F}(w^{t+1},\alpha^{t+1}) -\mathcal{F}(w^t,\alpha^t) && \nonumber \\
    & \leq  \eta_\alpha \mathcal{K}_1 +  (\eta_\alpha)^2\mathcal{K}_2  && \nonumber \\
    &\hspace{0.5cm}-((1-\xi)\eta_w  - \dfrac{L}{2}(\eta_w)^2) \| \nabla_w \mathcal{F}(w^t,\alpha^t)\|^2 + (\eta_w)^2 \mathcal{K}_3 \label{eq:theorem1boundF}
\end{flalign}

Summing up \eqref{eq:theorem1boundF} with $t = 1,\dots,T$ and rearranging the terms, we arrive at the following inequality:
\begin{flalign}
   & \sum_{t=1}^T ((1-\xi)\eta_w - \dfrac{L}{2}(\eta_w)^2)\| \nabla_w \mathcal{F}(w^t,\alpha^t)\|^2 && \nonumber \\
   & \leq \mathcal{F}(w^1,\alpha^1) - \mathcal{F}(w^{T+1},\alpha^{T+1})  && \nonumber \\
   &\hspace{0.5cm}+ \sum_{t=1}^T\eta_\alpha \mathcal{K}_1 + \sum_{t=1}^{T}(\eta_\alpha)^2 \mathcal{K}_2  + \sum_{t=1}^{T} (\eta_w)^2 \mathcal{K}_3 \label{eq:theorem2step2}
\end{flalign}
Furthermore, by definition of $\eta_\alpha$ and $\eta_w$ we can deduce that:
\begin{flalign*}
    & \min_{1\leq t\leq T} \|\nabla_w \mathcal{F}(w^t,\alpha^t) \|^2 &&\\
    &\leq \dfrac{\sum_{t=1}^T (\dfrac{\eta_w}{2} - \dfrac{L}{2}(\eta_w)^2) \| \nabla_w \mathcal{F}(w^t,\alpha^t)\|^2}{\sum_{t=1}^T ((1-\xi)\eta_w - \dfrac{L}{2}(\eta_w)^2)} &&\\
    & \leq \dfrac{1}{\sum_{t=1}^T ((1-\xi)\eta_w - \dfrac{L}{2}(\eta_w)^2)}( \mathcal{F}(w^1,\alpha^1) &&\\ 
    &\hspace{0.5cm}+ \sum_{t=1}^T\eta_\alpha \mathcal{K}_1 + \sum_{t=1}^{T}(\eta_\alpha)^2 \mathcal{K}_2 + \sum_{t=1}^{T}(\eta_w)^2 \mathcal{K}_3)   &&\\
    & \leq \dfrac{1}{\sum_{t=1}^{T} \dfrac{(1-\xi)\eta_w}{2}}(\mathcal{F}(w^1,\alpha^1) + T\eta_\alpha \mathcal{K}_1 &&\\
    &\hspace{0.5cm}+  T L \eta_\alpha \mathcal{K}_2+ \sum_{t=1}^{T}(\eta_w)^2 \mathcal{K}_3) && \\
    & \leq \dfrac{2\mathcal{F}(w^1,\alpha^1)}{T \eta_w(1-\xi)} + \dfrac{2T\eta_\alpha \mathcal{K}_1}{T \eta_w(1-\xi)} &&\\
    &\hspace{0.5cm}+ \dfrac{T L \eta_\alpha \mathcal{K}_2}{T \eta_w(1-\xi)} +  \dfrac{2\mathcal{K}_3}{1-\xi} \dfrac{\sum_{t=1}^{T} (\eta_w)^2}{\sum_{t=1}^{T} \eta_w}
\end{flalign*}
The third inequality holds because $(1-\xi)\eta_w - \dfrac{L}{2} (\eta_w)^2 \geq \dfrac{(1-\xi)\eta_w}{2}$\\
\\
Using $\eta_\alpha = \dfrac{c}{T}$ and $\dfrac{\sum_{t=1}^{T} (\eta_w)^2}{\sum_{t=1}^{T} \eta_w} = \eta_w = \dfrac{C}{\sqrt{T}} $, it leads to:\\
\begin{flalign*}
    & \min_{1\leq t \leq T} \| \nabla_w \mathcal{F}(w^t,\alpha^t) \|^2 &&\\
    & \leq \dfrac{2\mathcal{F}(w^1,\alpha^1)}{C\sqrt{T}(1-\xi)} + \dfrac{2c(M^2 + M^2(2\delta+1))}{C\sqrt{T}(1-\xi)} &&\\
    &\hspace{0.5cm}+ \dfrac{ L c(M+M(2\delta +1))^2}{C\sqrt{T}(1-\xi)} 
    + \dfrac{2\mathcal{K}}{1-\xi}\dfrac{C}{\sqrt{T}} &&\\
    &\leq \mathcal{O}\left( \dfrac{H}{\sqrt{T}} \right)
\end{flalign*}

\end{IEEEproof}

\subsection{Proof of Theorem 2}
\label{proof2}


\begin{IEEEproof}
First, we reuse on the above \eqref{eq:theorem2step2} and let the number of step $T \to \infty$
\begin{flalign}
   & \sum_{t=1}^T ((1-\xi)\eta_w^{(t)} - \dfrac{L}{2}(\eta_w^{(t)})^2)\| \nabla_w \mathcal{F}(w^t,\alpha^t)\|^2 && \nonumber \\
   & \leq \mathcal{F}(w^{1},\alpha^{1}) -\mathcal{F}(w^{T+1},\alpha^{T+1}) && \nonumber \\
   &\hspace{0.5cm}+ \sum_{t=1}^T\eta_\alpha^{(t)} \mathcal{K}_1 + \sum_{t=1}^{T}(\eta_\alpha^{(t)})^2 \mathcal{K}_2  + \sum_{t=1}^{T} (\eta_w^{(t)})^2 \mathcal{K}_3 
\end{flalign}

Because $\eta^{(t)}_\alpha$ is a convergent series
we have the infinite sum of both sides:
\begin{flalign*}
& \sum_{t=1}^\infty \frac{(1 - \xi) \eta_w^{(t)}}{2} \left\| \nabla_w \mathcal{F}(w^t, \alpha^t) \right\|^2 &&\\
&\leq 
\sum_{t=1}^\infty \left( (1 - \xi)\eta_w^{(t)} - \frac{L}{2}(\eta_w^{(t)})^2 \right) \left\| \nabla_w \mathcal{F}(w^t, \alpha^t) \right\|^2 &&\\
&< \infty
\end{flalign*}

Since $\sum_{t=1}^\infty \frac{(1 - \xi)\eta_w^{(t)}}{2} = \infty$, according to Lemma \ref{lemma3}, it only requires to prove:
\begin{flalign}
    \big|\| \nabla_w \mathcal{F}(w^{t+1},\alpha^{t+1})\|^2 - \| \nabla_w \mathcal{F}(w^t,\alpha^t)\|^2 \big| \leq K\eta_w^{(t)}
    \label{theorem2:eq1}
\end{flalign}
Using Lemma \ref{lemma2} and both Assumptions \ref{assumption1}, \ref{assumption2}, we can obtain
\begin{flalign}
    &\big|\| \nabla_w \mathcal{F}(w^{t+1},\alpha^{t+1})\|^2 - \| \nabla_w \mathcal{F}(w^t,\alpha^t)\|^2 \big| \nonumber \\
    &= \big|\| \nabla_w \mathcal{F}(w^{t+1},\alpha^{t+1})\|+ \| \nabla_w \mathcal{F}(w^t,\alpha^t)\| \big| \times && \nonumber\\
    &\hspace{0.5cm}\big|\| \nabla_w \mathcal{F}(w^{t+1},\alpha^{t+1})\| - \| \nabla_w \mathcal{F}(w^t,\alpha^t)\| \big| \nonumber\\
    &\leq 2 M L\|(w^{t+1},\alpha^{t+1}) -(w^t,\alpha^t)\| && \nonumber\\
    &= 2 M L \sqrt{\|w^{t+1}-w^t\|^2 + \|\alpha^{t+1}-\alpha^t \|^2} \label{theorem2:eq2}
\end{flalign}
We further bound step size of $w$ and $\alpha$ as:
\begin{flalign}
    &\|w^{t+1}-w^t\|^2 + \|\alpha^{t+1}-\alpha^t \|^2 && \nonumber\\
    &= (\eta_w^{(t)})^2\|\nabla_w \mathcal{F}(w^t,\alpha^t) + \xi \bar{\beta}_t \nabla_w \mathcal{Q}(w^t,\alpha^t)\|^2 && \nonumber\\
    &\hspace{0.5cm}+(\eta_\alpha^{(t)})^2 \|\nabla_\alpha \mathcal{F}(w^t,\alpha^t) + \xi \bar{\beta}_t \nabla_\alpha \mathcal{Q}(w^t,\alpha^t)\|^2 && \nonumber\\
    &\leq (\eta_w^{(t)})^2 (M+M(2\delta+1))^2 + (\eta_w^{(t)})^2 (M+M(2\delta+1))^2 \label{theorem2:eq3}
\end{flalign}
Combine the esitmation of \eqref{theorem2:eq2} and \eqref{theorem2:eq3}, we have the inequality in \eqref{theorem2:eq1}.
As a result, $\| \nabla_w F(w^{(t)},\alpha^{(t)})\| ^2 \to 0$, lead to $\| \nabla_w F(w^{(t)},\alpha^{(t)})\|  \to 0$.
\end{IEEEproof}


\end{appendix}

%% file: references.bib
@article{yan2014,
  title={Learning from multiple annotators with varying expertise},
  author={Yan, Yan and Rosales, R{\'o}mer and Fung, Glenn and Subramanian, Ramanathan and Dy, Jennifer},
  journal={Machine learning},
  volume={95},
  number={3},
  pages={291--327},
  year={2014},
  publisher={Springer}
}

@inproceedings{arpit2017,
  title={A closer look at memorization in deep networks},
  author={Arpit, Devansh and Jastrz{\k{e}}bski, Stanis{\l}aw and Ballas, Nicolas and Krueger, David and Bengio, Emmanuel and Kanwal, Maxinder S and Maharaj, Tegan and Fischer, Asja and Courville, Aaron and Bengio, Yoshua and others},
  booktitle={International conference on machine learning},
  pages={233--242},
  year={2017},
  organization={PMLR}
}

@article{blum2003,
  title={Noise-tolerant learning, the parity problem, and the statistical query model},
  author={Blum, Avrim and Kalai, Adam and Wasserman, Hal},
  journal={Journal of the ACM (JACM)},
  volume={50},
  number={4},
  pages={506--519},
  year={2003},
  publisher={ACM New York, NY, USA}
}

@article{shu2019,
  title={Meta-weight-net: Learning an explicit mapping for sample weighting},
  author={Shu, Jun and Xie, Qi and Yi, Lixuan and Zhao, Qian and Zhou, Sanping and Xu, Zongben and Meng, Deyu},
  journal={Advances in neural information processing systems},
  volume={32},
  year={2019}
}

@article{hendrycks2018,
  title={Using trusted data to train deep networks on labels corrupted by severe noise},
  author={Hendrycks, Dan and Mazeika, Mantas and Wilson, Duncan and Gimpel, Kevin},
  journal={Advances in neural information processing systems},
  volume={31},
  year={2018}
}

@inproceedings{zheng2021,
  title={Meta label correction for noisy label learning},
  author={Zheng, Guoqing and Awadallah, Ahmed Hassan and Dumais, Susan},
  booktitle={Proceedings of the AAAI conference on artificial intelligence},
  volume={35},
  number={12},
  pages={11053--11061},
  year={2021}
}

@article{liu2022,
  title={Bome! bilevel optimization made easy: A simple first-order approach},
  author={Liu, Bo and Ye, Mao and Wright, Stephen and Stone, Peter and Liu, Qiang},
  journal={Advances in neural information processing systems},
  volume={35},
  pages={17248--17262},
  year={2022}
}

@inproceedings{patrini2017,
  title={Making deep neural networks robust to label noise: A loss correction approach},
  author={Patrini, Giorgio and Rozza, Alessandro and Krishna Menon, Aditya and Nock, Richard and Qu, Lizhen},
  booktitle={Proceedings of the IEEE conference on computer vision and pattern recognition},
  pages={1944--1952},
  year={2017}
}

@inproceedings{jiang2018,
  title={Mentornet: Learning data-driven curriculum for very deep neural networks on corrupted labels},
  author={Jiang, Lu and Zhou, Zhengyuan and Leung, Thomas and Li, Li-Jia and Fei-Fei, Li},
  booktitle={International conference on machine learning},
  pages={2304--2313},
  year={2018},
  organization={PMLR}
}

@article{han2018,
  title={Co-teaching: Robust training of deep neural networks with extremely noisy labels},
  author={Han, Bo and Yao, Quanming and Yu, Xingrui and Niu, Gang and Xu, Miao and Hu, Weihua and Tsang, Ivor and Sugiyama, Masashi},
  journal={Advances in neural information processing systems},
  volume={31},
  year={2018}
}

@inproceedings{yu2019,
  title={How does disagreement help generalization against label corruption?},
  author={Yu, Xingrui and Han, Bo and Yao, Jiangchao and Niu, Gang and Tsang, Ivor and Sugiyama, Masashi},
  booktitle={International conference on machine learning},
  pages={7164--7173},
  year={2019},
  organization={PMLR}
}

@inproceedings{wang2021,
  title={Proselflc: Progressive self label correction for training robust deep neural networks},
  author={Wang, Xinshao and Hua, Yang and Kodirov, Elyor and Clifton, David A and Robertson, Neil M},
  booktitle={Proceedings of the IEEE/CVF conference on computer vision and pattern recognition},
  pages={752--761},
  year={2021}
}

@inproceedings{ren2018,
  title={Learning to reweight examples for robust deep learning},
  author={Ren, Mengye and Zeng, Wenyuan and Yang, Bin and Urtasun, Raquel},
  booktitle={International conference on machine learning},
  pages={4334--4343},
  year={2018},
  organization={PMLR}
}

@ARTICLE{zhao2021,
  author={Zhao, Qian and Shu, Jun and Yuan, Xiang and Liu, Ziming and Meng, Deyu},
  journal={IEEE Transactions on Neural Networks and Learning Systems}, 
  title={A Probabilistic Formulation for Meta-Weight-Net}, 
  year={2023},
  volume={34},
  number={3},
  pages={1194-1208},
  doi={10.1109/TNNLS.2021.3105104}}

@article{liu2023,
  title={Value-function-based sequential minimization for bi-level optimization},
  author={Liu, Risheng and Liu, Xuan and Zeng, Shangzhi and Zhang, Jin and Zhang, Yixuan},
  journal={IEEE Transactions on Pattern Analysis and Machine Intelligence},
  volume={45},
  number={12},
  pages={15930--15948},
  year={2023},
  publisher={IEEE}
}

@article{sow2022,
  title={A constrained optimization approach to bilevel optimization with multiple inner minima},
  author={Sow, Daouda and Ji, Kaiyi and Guan, Ziwei and Liang, Yingbin},
  journal={arXiv preprint arXiv:2203.01123},
  pages={4},
  year={2022}
}

@article{song2022,
  title={Learning from noisy labels with deep neural networks: A survey},
  author={Song, Hwanjun and Kim, Minseok and Park, Dongmin and Shin, Yooju and Lee, Jae-Gil},
  journal={IEEE transactions on neural networks and learning systems},
  volume={34},
  number={11},
  pages={8135--8153},
  year={2022},
  publisher={IEEE}
}

@article{li2021survey,
  title={A survey of convolutional neural networks: analysis, applications, and prospects},
  author={Li, Zewen and Liu, Fan and Yang, Wenjie and Peng, Shouheng and Zhou, Jun},
  journal={IEEE transactions on neural networks and learning systems},
  volume={33},
  number={12},
  pages={6999--7019},
  year={2021},
  publisher={IEEE}
}

@article{resnet32,
  author       = {Kaiming He and
                  Xiangyu Zhang and
                  Shaoqing Ren and
                  Jian Sun},
  title        = {Deep Residual Learning for Image Recognition},
  journal      = {CoRR},
  volume       = {abs/1512.03385},
  year         = {2015},
  url          = {http://arxiv.org/abs/1512.03385},
  eprinttype    = {arXiv},
  eprint       = {1512.03385},
  bibsource    = {dblp computer science bibliography, https://dblp.org}
}

@article{dmlp2023,
  title={Learning from Noisy Labels with Decoupled Meta Label Purifier},
  author={Yuanpeng Tu and Boshen Zhang and Yuxi Li and Liang Liu and Jian Li and Jiangning Zhang and Yabiao Wang and Chengjie Wang and Cai Rong Zhao},
  journal={2023 IEEE/CVF Conference on Computer Vision and Pattern Recognition (CVPR)},
  year={2023},
  pages={19934-19943},
  url={https://api.semanticscholar.org/CorpusID:256846984}
}

@article{emlc2023,
  title={Enhanced Meta Label Correction for Coping with Label Corruption},
  author={Mitchell Keren Taraday and Chaim Baskin},
  journal={2023 IEEE/CVF International Conference on Computer Vision (ICCV)},
  year={2023},
  pages={16249-16258},
  url={https://api.semanticscholar.org/CorpusID:258832322}
}

@article{dividemix2020,
  title={DivideMix: Learning with Noisy Labels as Semi-supervised Learning},
  author={Junnan Li and Richard Socher and Steven C. H. Hoi},
  journal={ArXiv},
  year={2020},
  volume={abs/2002.07394},
  url={https://api.semanticscholar.org/CorpusID:211146562}
}

@inproceedings{taraday2023emlc,
  author    = {Taraday, Mitchell Keren and Baskin, Chaim},
  title     = {Enhanced Meta Label Correction for Coping with Label Corruption},
  booktitle = {Proceedings of the IEEE/CVF International Conference on Computer Vision (ICCV)},
  year      = {2023},
  pages     = {16295--16304}
}

@inproceedings{tu2023dmlp,
  author    = {Tu, Yuanpeng and Zhang, Boshen and Li, Yuxi and Liu, Liang and Li, Jian and Wang, Yabiao and Wang, Chengjie and Zhao, Cai Rong},
  title     = {Learning From Noisy Labels With Decoupled Meta Label Purifier},
  booktitle = {Proceedings of the IEEE/CVF Conference on Computer Vision and Pattern Recognition (CVPR)},
  year      = {2023},
  pages     = {19934--19943}
}

@article{kim2025splitnet,
  title   = {SplitNet: Learnable Clean-Noisy Label Splitting for Learning with Noisy Labels},
  author  = {Kim, Daehwan and Ryoo, Kwangrok and Cho, Hansang and Kim, Seungryong},
  journal = {International Journal of Computer Vision},
  volume  = {133},
  pages   = {549--566},
  year    = {2025},
  doi     = {10.1007/s11263-024-02187-4}
}

@inproceedings{xu2025dulc,
  title     = {Revisiting Interpolation for Noisy Label Correction},
  author    = {Xu, Yuanzhuo and Niu, Xiaoguang and Yang, Jie and Su, Ruiyi and Zhang, Jian and Liu, Shubo and Drew, Steve},
  booktitle = {Proceedings of the AAAI Conference on Artificial Intelligence},
  year      = {2025},
  volume    = {39},
  number    = {20},
  pages     = {21833--21841},
  doi       = {10.1609/aaai.v39i20.35489}
}

@article{algan2021survey,
  title   = {Image classification with deep learning in the presence of noisy labels: A survey},
  author  = {Algan, G{\"o}rkem and Ulusoy, Ilkay},
  journal = {Knowledge-Based Systems},
  volume  = {215},
  pages   = {106771},
  year    = {2021},
  doi     = {10.1016/j.knosys.2021.106771}
}

@inproceedings{jiang2018mentornet,
  title     = {MentorNet: Learning Data-Driven Curriculum for Very Deep Neural Networks on Corrupted Labels},
  author    = {Jiang, Lu and Zhou, Zhengyuan and Leung, Thomas and Li, Li-Jia and Fei-Fei, Li},
  booktitle = {Proceedings of the 35th International Conference on Machine Learning (ICML)},
  year      = {2018}
}

@inproceedings{malach2017decoupling,
  title     = {Decoupling ``When to Update'' from ``How to Update''},
  author    = {Malach, Eran and Shalev-Shwartz, Shai},
  booktitle = {Advances in Neural Information Processing Systems (NeurIPS)},
  year      = {2017}
}

@inproceedings{reed2015bootstrapping,
  title     = {Training Deep Neural Networks on Noisy Labels with Bootstrapping},
  author    = {Reed, Scott and Lee, Honglak and Anguelov, Dragomir and Szegedy, Christian and Erhan, Dumitru and Rabinovich, Andrew},
  booktitle = {ICLR Workshop},
  year      = {2015}
}

@inproceedings{zhang2018gce,
  title     = {Generalized Cross Entropy Loss for Training Deep Neural Networks with Noisy Labels},
  author    = {Zhang, Zhilu and Sabuncu, Mert R.},
  booktitle = {Advances in Neural Information Processing Systems (NeurIPS)},
  year      = {2018}
}

@inproceedings{wang2019symmetric,
  title     = {Symmetric Cross Entropy for Robust Learning with Noisy Labels},
  author    = {Wang, Yisen and Ma, Xingjun and Chen, Zaiyi and Luo, Yuan and Yi, Jinfeng and Bailey, James},
  booktitle = {Proceedings of the IEEE/CVF International Conference on Computer Vision (ICCV)},
  year      = {2019}
}

@inproceedings{liu2020elr,
  title     = {Early-Learning Regularization Prevents Memorization of Noisy Labels},
  author    = {Liu, Sheng and Niles-Weed, Jonathan and Razavian, Narges and Fernandez-Granda, Carlos},
  booktitle = {Advances in Neural Information Processing Systems (NeurIPS)},
  year      = {2020}
}
